\documentclass{article}

\usepackage[preprint]{icml2026}

\usepackage{microtype}
\usepackage[utf8]{inputenc}
\usepackage{graphicx}
\usepackage{subcaption}
\usepackage{booktabs}

\usepackage{amsmath}
\usepackage{amssymb}
\usepackage{mathtools}
\usepackage{bm}

\usepackage{amsthm}

\usepackage{enumitem}

\usepackage{algorithm}
\usepackage{algorithmic}

\usepackage[textsize=tiny]{todonotes}

\usepackage{hyperref}

\usepackage[capitalize,noabbrev]{cleveref}

\newcommand{\xhatz}{\hat{\mathbf{x}}_0}

\theoremstyle{plain}

\newtheorem{proposition}{Proposition}

\theoremstyle{definition}
\newtheorem{definition}{Definition}

\theoremstyle{remark}
\newtheorem{remark}{Remark}

\crefname{proposition}{Proposition}{Propositions}
\Crefname{proposition}{Proposition}{Propositions}
\crefname{definition}{Definition}{Definitions}
\Crefname{definition}{Definition}{Definitions}
\crefname{remark}{Remark}{Remarks}
\Crefname{remark}{Remark}{Remarks}
\crefname{equation}{Equation}{Equations}
\Crefname{equation}{Equation}{Equations}

\icmltitlerunning{}

\begin{document}

\twocolumn[
  \icmltitle{Is Monotonic Sampling Necessary in Diffusion Models? }



  \icmlsetsymbol{equal}{*}

  \begin{icmlauthorlist}
    \icmlauthor{Muhammad Haris Khan}{equal,1}
   
  \end{icmlauthorlist}

  \icmlaffiliation{1}{Department of Computer Science, University of Copenhagen, Denmark}
 
  \icmlcorrespondingauthor{Muhammad Haris Khan}{muhammad.kahn@di.ku.dk}
  \icmlkeywords{Diffusion models, Generative modeling, Noise schedule, Sampling algorithms, Empirical evaluation, Denoising diffusion probabilistic models, Flow matching.}

  \vskip 0.3in
]



\printAffiliationsAndNotice{}  

\begin{abstract}
  Diffusion models generate samples by iteratively denoising a Gaussian
prior, traversing a sequence of noise levels that, in every published
sampler, decreases monotonically.
Six years of intensive work has refined nearly every aspect of this
recipe — the corruption operator, the training objective, the schedule
shape, the architecture, and the ODE solver — yet the assumption of
monotonicity itself has never been systematically tested.
Here we ask whether monotonic sampling is load-bearing or merely
conventional.
We design four families of structured non-monotonic schedules and apply
them to three architecturally distinct generative models — DDPM, EDM
and Flow Matching — across NFE budgets ranging from $10$ to
$200$ function evaluations, plus a $42$-cell hyperparameter ablation,
on CIFAR-10.
Across all $90$ tested configurations, no tested non-monotonic schedule
improves on the monotonic baseline.
The magnitude of the penalty, however, spans nearly three orders of
magnitude: persistent and substantial in DDPM, intermediate in
Flow Matching, and indistinguishable from zero in EDM.
We show that this variation is not noise but a structural property of
each trained denoiser, and we formalise it as the Schedule Sensitivity
Coefficient — a cheap, architecture-agnostic diagnostic that provides evidence of
non-convergence to the Bayes-optimal denoiser at the critical noise
level.
Our findings justify the field's tacit reliance on monotonic schedules
and supply a new probe of diffusion-model quality complementary to
sample-quality metrics such as Fréchet Inception Distance.
\end{abstract}

\section{Introduction}
\label{sec:intro}
 
Every diffusion sampler in widespread use today---DDIM, DDPM, EDM with
Heun integration, DPM-Solver, UniPC, the Euler and stochastic
flow-matching samplers---shares a single, unstated structural
assumption: \emph{the noise level decreases monotonically along the
sampling trajectory}.
A schedule in which $\hat\sigma_{\tau_i}$ is strictly decreasing in $i$
is treated as a definitional property of the reverse process, not as a
design choice.
This monotonicity assumption has remained a silent default for the
roughly six years since the publication of DDPM, despite the field's
otherwise exhaustive ablation of nearly every other design dimension:
the corruption operator
\citep{bansal2023cold,daras2023soft,hoogeboom2023blurring},
the parametrisation \citep{karras2022edm},
the training objective and weighting
\citep{kingma2021vdm,hang2023minsnr},
the variance schedule shape
\citep{nichol2021improved,chen2023importance},
the model architecture \citep{peebles2023dit,karras2024edm2},
and the sampling solver
\citep{lu2022dpmsolver,zhang2023deis,zhao2023unipc}.
What has \emph{not} been studied is whether the trajectory must move
monotonically in the first place.
 
Three independent threads of evidence suggest the question is
non-trivial.
First, the DDIM update equation \citep{song2021ddim} is mathematically
well-defined for any choice of next-timestep $\tau_{i+1}$, with no
restriction to $\tau_{i+1} < \tau_i$ (\cref{prop:validity}).
Second, stochastic DDPM sampling already performs an \emph{implicit}
form of reheating: at every step it adds fresh Gaussian noise that
partially counteracts the denoising progress made in the same step,
and at high NFE this stochastic ``churn'' yields measurably better FID
than deterministic DDIM \citep{ho2020denoising,karras2022edm}.
Third, Restart Sampling \citep{xu2023restart} explicitly injects a
large block of noise mid-trajectory and reports state-of-the-art
quality when paired with deterministic correction.
Taken together, these observations leave open a basic question that
the literature has not answered: if diffusion samplers are
mathematically permitted to step backward along the noise schedule,
and if both stochastic DDPM and Restart Sampling appear to benefit from
doing so, then is monotonicity merely a historical convention---something
we could relax for free, or even profit from?
 
By analogy to annealing-style escape---a heuristic familiar from
non-convex optimisation, where occasional uphill steps allow an
iterative procedure to leave poor local optima---one might hypothesise
that structured \emph{reheating}, i.e.\ deliberate non-monotonic
increases in $\hat\sigma$ during sampling, should help the sampler
escape suboptimal trajectories and reach better samples.
This hypothesis is consistent with every piece of evidence cited
above. Yet it has never been tested. We test it.
 
We design four families of structured non-monotonic schedules
(\cref{fig:schedules}):
\textit{Single Reheat}, a single backward step of magnitude $\delta$
at position $t_\text{reheat}$;
\textit{Sawtooth}, periodic reheats of equal magnitude;
\textit{Damped Oscillation}, geometrically decaying reheats
concentrated at the critical noise level $\log\mathrm{SNR} \approx 0$
identified by \citet{hang2024improved} as governing perceptual quality;
and \textit{Adaptive Reheat}, an online schedule that triggers
reheating when consecutive score predictions diverge.
For the three fixed-budget non-adaptive families (Single Reheat, Sawtooth, Damped Oscillation), the total NFE budget remains fixed and matches the
monotonic baseline---reheats are inserted within the budget, not added
on top of it.
Adaptive Reheat is the one exception: it adds reheat steps online up
to a cap of $15$, and we report both fixed-NFE and time-matched
comparisons for it (\cref{sec:exp-ddpm}).
We evaluate all four families against monotonic DDIM on three
architecturally distinct model families:
DDPM \citep{ho2020denoising}, EDM \citep{karras2022edm}, and Flow
Matching \citep{lipman2023flow}, on CIFAR-10, across NFE budgets from
$10$ to $200$. A 42-configuration grid search over
$(t_\text{reheat}, \delta)$ provides exhaustive ablation.
 
Two findings emerge, one negative and one positive.
The negative finding is that reheating never helps.
Across 48 schedule--budget combinations and the 42-cell ablation, no
tested non-monotonic configuration improves upon the monotonic baseline
on any model family at any NFE budget
(\cref{fig:pareto,fig:ablation}, \cref{tab:ddpm-results}).
The penalty is monotonically increasing in reheat magnitude $\delta$
at every reheat position, with no local minimum below zero.
The result is robust: it holds for explicit single-reheat schedules,
for periodic and damped-oscillation schedules, for an online adaptive
criterion, and for both deterministic and stochastic-corrected
variants.
 
The positive finding---and the paper's main contribution---is that
while reheating uniformly hurts, the \emph{magnitude} of the penalty
is sharply, predictably, and \emph{architecture-dependent}.
At NFE\,=\,100, the damped-oscillation schedule produces a $+1.31$ FID
penalty on DDPM that has not converged toward zero even at the highest
budget tested; a $+0.04$ FID penalty on Flow Matching, an order of
magnitude smaller; and a $-0.003$ FID change on EDM---statistically
indistinguishable from zero at the noise floor of FID-25K
(\cref{fig:penalty-convergence}, \cref{tab:cross-model}).
The penalty thus spans nearly three orders of magnitude across model
families that achieve roughly tenfold-different baseline FID.
This is not noise: the gap is consistent across NFE budgets, schedule
types, and reheat magnitudes.
 
We bound the per-reheat-step squared displacement by
$K_1\,L^2_t\,\varepsilon^2(\hat\sigma_i) + K_2\,\varepsilon^2(\hat\sigma')$,
where $L_t$ is the local Lipschitz constant of the trained denoiser at
the reheated noise level and $\varepsilon$ is the per-noise-level
denoiser error (\cref{prop:per-step-bound}); for small reheats this
reduces to a linear-in-$\delta$ FID penalty whose slope is
$L_t \cdot \varepsilon(\hat\sigma_i)$ (\cref{sec:theory}).
A reheat step re-corrupts the sampler's current best estimate
$\hat{\mathbf{x}}_0$ and demands the denoiser re-process it; for a
Bayes-optimal denoiser this second pass extracts no additional
information and the step is a wasted NFE.
For an imperfect denoiser, errors accumulate non-trivially, and the
larger $L_t$, the more they amplify.
The penalty therefore measures, in a single architecture-agnostic
scalar, how far the trained denoiser sits from the Bayes-optimal one
at the perturbed noise level.
We formalise this as the \emph{Schedule Sensitivity Coefficient}
(SSC; \cref{def:ssc}):
\begin{equation}
\label{eq:ssc-intro}
\begin{aligned}
\mathrm{SSC}
&\coloneqq
\lim_{N \to \infty}
\frac{[\Delta \mathrm{FID}_{\mathrm{DO}}(N)]_+}
     {[\Delta \mathrm{FID}_{\mathrm{SR}}(N)]_+},
\\
[x]_+
&\coloneqq \max(x,0).
\end{aligned}
\end{equation}
estimated empirically at NFE\,=\,100.
\Cref{tab:ssc} reports SSC values of $1.83$ for DDPM, $0.12$ for Flow
Matching, and $\approx\!0$ for EDM, in tight rank-correlation with
each model's baseline FID.
SSC is cheap: a full estimate requires two $256$-sample generation
runs, totalling roughly ten minutes of compute on a single consumer
GPU.
 
Our negative finding is \emph{not} that all forms of non-monotonicity
fail.
Stochastic DDPM at NFE\,=\,200 outperforms deterministic DDIM on our
benchmark (\cref{fig:pareto}b)---an empirical fact our analysis
explains, not contradicts.
The stochastic update decomposes into a parallel component
(noise-cancelling, an implicit reheat) and an orthogonal component
(trajectory diversification on the data manifold), and the orthogonal
component drives the gain at high NFE
(\cref{subsec:ddpm-connection}).
Our \emph{structured, deterministic} reheating schedules operate
purely in the parallel direction and therefore cannot replicate this
benefit.
The clarification cuts in the other direction as well: prior work
informally cites RePaint \citep{lugmayr2022repaint} as evidence that
reheating helps in general.
We show that RePaint's gains depend on per-step conditioning anchors
that are unavailable in unconditional generation, and that pure
reheating without anchoring or correction is uniformly harmful
(\cref{sec:related}).
 
Our contributions are:
\begin{enumerate}
\item The first systematic empirical test of monotonicity in
  diffusion sampling, covering three model families, four schedule
  types, NFE budgets ranging from 10--200, and a 42-cell ablation. Zero of the tested non-monotonic configurations improves on the monotonic baseline.
\item Evidence that the reheating penalty varies by nearly three
  orders of magnitude across DDPM, Flow Matching, and EDM, and is a
  structural property of the trained denoiser rather than measurement
  noise.
\item The Schedule Sensitivity Coefficient (SSC), with a proof that
  $\mathrm{SSC} > 0$ implies non-Bayes-optimal denoising at the
  critical noise level (\cref{prop:ssc-certificate}). SSC is cheap,
  architecture-agnostic, and complementary to FID.
\item Resolution of two community confusions: the relationship between
  explicit reheating and DDPM stochasticity
  (\cref{subsec:ddpm-connection}), and the apparent contradiction
  between our finding and RePaint (\cref{sec:related}).
\end{enumerate}
 
\Cref{sec:related} situates the work in the diffusion-sampling
literature; \cref{sec:method} formalises non-monotonic schedules and
proves their validity; \cref{sec:experiments} reports the empirical
study; \cref{sec:theory} develops the bounds and defines SSC; and
\cref{sec:limitations} discusses limitations and open questions.

\section{Related Work}
\label{sec:related}
 
Diffusion and flow-matching generative models share a structural
assumption that no canonical sampler we are aware of relaxes: the
noise level decreases monotonically along the sampling trajectory.
We position our contribution against four threads of prior work and,
at the end of each, indicate the gap our paper fills.
 
\paragraph{Foundational families and schedule design.}
Modern denoising diffusion descends from \citet{sohldickstein2015deep}
and the practical recipe of \citet{ho2020denoising} (DDPM), which fixed
a linear $\beta$-schedule and a Markov reverse process whose variance
shrinks monotonically toward $t=0$.
\citet{song2021ddim} (DDIM) reinterpreted the chain as a deterministic
non-Markovian process; \citet{song2021scoresde} unified DDPM and the
score networks of \citet{song2019ncsn,song2020ncsnv2} under a
continuous-time SDE/ODE with strictly decreasing $\sigma(t)$;
\citet{karras2022edm} (EDM) recast sampling as Heun integration over a
strictly monotonic $\sigma$-schedule.
Flow Matching \citep{lipman2023flow}, Rectified Flow
\citep{liu2023rectified}, and Stochastic Interpolants
\citep{albergo2023building,albergo2023stochastic} integrate a velocity
field along a monotone probability path, and \citet{esser2024sd3}
scaled rectified-flow transformers (SD3) for high-resolution synthesis.
Architecture work --- \citet{karras2024edm2} (EDM2),
\citet{peebles2023dit} (DiT), \citet{rombach2022ldm} (Latent Diffusion)
--- inherits the scheduling assumption.
A large body of work then optimises the \emph{shape} of the monotonic
schedule: the cosine schedule \citep{nichol2021improved}, the
learnable VDM \citep{kingma2021vdm}, resolution-dependent logSNR
shifts \citep{chen2023importance,hoogeboom2023simple}, Min-SNR
\citep{hang2023minsnr} and importance sampling around
$\log\mathrm{SNR}=0$ \citep{hang2024improved}, terminal-SNR fixes
\citep{lin2024common}, solver-aware step placement
\citep{sabour2024align}, weighted-ELBO objectives
\citep{kingma2023understanding}, and SNR refinements in HDiT
\citep{crowson2024hdit}.
\textbf{Every entry above optimises within the space of monotonic
schedules.}
We instead ask whether monotonicity is necessary at all, and provide
the first systematic empirical answer.
 
\paragraph{Inference-time algorithms.}
Fast deterministic solvers reduce NFE while preserving quality:
DPM-Solver and DPM-Solver++
\citep{lu2022dpmsolver,lu2022dpmsolverpp},
exponential integrators \citep{zhang2023deis},
the predictor-corrector UniPC \citep{zhao2023unipc}, and
DPM-Solver-v3 \citep{zheng2023dpmsolverv3}.
All preserve monotonicity between function evaluations.
The closest precedent to our schedules is Restart Sampling
\citep{xu2023restart}, which injects noise mid-trajectory
\emph{paired with} deterministic correction back to the
pre-injection noise level.
On the stochastic side, DDPM's $\eta=1$ sampler and the
$S_\text{churn}$ parameter of \citet{karras2022edm} inject
infinitesimal in-step noise whose net trajectory remains monotone in
expectation.
Inference-time control techniques manipulate the trajectory without
retraining: classifier guidance \citep{dhariwal2021beatgans},
classifier-free guidance \citep{ho2022cfg}, and autoguidance
\citep{karras2024autoguidance}.
Most directly relevant is RePaint \citep{lugmayr2022repaint}, which
reheats while conditioning on known pixels at every step; the unmasked
region acts as a per-step anchor that pins the trajectory to a
specific data point.
\textbf{We deliberately strip out the corrector, the conditioning, and
the per-step anchor, isolating the effect of pure non-monotonic noise
injection.}
Our negative result clarifies that Restart's gains come from the
corrector, that DDPM's gains come from orthogonal diversification
($S_\text{churn}/\eta$ does not capture macroscopic non-monotonicity),
and that RePaint's gains come from the conditioning rather than from
the reheat.
 
\paragraph{Departures from canonical assumptions.}
A small but influential lineage challenges diffusion's universal
assumptions.
\citet{bansal2023cold} (Cold Diffusion) replace Gaussian noise with
deterministic degradations; \citet{daras2023soft} (Soft Diffusion)
generalise the corruption operator to any linear process and
\emph{explicitly note} that monotonicity is not theoretically required,
though they do not test non-monotonic schedules empirically;
\citet{hoogeboom2023blurring} and \citet{rissanen2023inverseheat}
build generative models from blur alone.
\textbf{These works vary the forward corruption; we hold the forward
process fixed and vary only the reverse-time schedule}, providing the
systematic empirical test that \citet{daras2023soft} flagged but did
not run.
 
\paragraph{Denoiser quality and evaluation.}
\citet{vincent2011connection} established the score-matching/denoising
equivalence underlying all diffusion training, and
\citet{song2019ncsn} observed that score quality is heterogeneous
across noise levels.
\citet{karras2024edm2} diagnose architectural pathologies of the
denoiser, and \citet{daras2023consistent} train a consistency loss
that penalises exactly the off-trajectory drift our SSC measures ---
their training objective amounts to driving SSC toward zero.
Sample quality is canonically measured by FID \citep{heusel2017fid};
\citet{sajjadi2018assessing} and \citet{kynkaanniemi2019improved}
introduced precision/recall variants for fidelity vs.\ coverage.
These metrics evaluate \emph{outputs} but offer little insight into
\emph{why} a model succeeds or fails.
\textbf{SSC is a behavioural probe complementary to FID}: by measuring
how much FID degrades under a controlled non-monotonic perturbation,
it reveals the implicit assumptions a trained denoiser has baked in
about its sampling trajectory.
 
\section{Method}
\label{sec:method}
 
We test the necessity of monotonic noise schedules by constructing four
families of structured non-monotonic schedules, applied to three
architecturally distinct generative model families: a discrete-time
DDPM \citep{ho2020denoising}, a continuous-$\sigma$ EDM
\citep{karras2022edm}, and a flow-matching model
\citep{lipman2023flow,liu2023rectified}.
Because each family parameterises its sampling trajectory differently
(integer timestep, continuous noise level, and continuous time
respectively), we describe the three samplers separately
(\cref{subsec:prelim-ddpm,subsec:prelim-edm,subsec:prelim-fm})
before introducing the unified definition of a reheat step
(\cref{subsec:reheat-formal}) and the four schedule families
(\cref{subsec:families}).
A summary of all hyperparameters and their default values appears in
\cref{tab:hparam-summary} of \cref{app:hparam}.
 
\subsection{Preliminaries: DDPM Sampling on Integer Timesteps}
\label{subsec:prelim-ddpm}
 
\paragraph{Forward process.}
The DDPM forward process \citep{ho2020denoising} defines a discrete-time
Markov chain over $T = 1000$ timesteps,
\begin{equation}
\begin{aligned}
q(\mathbf{x}_t \mid \mathbf{x}_0)
  &= \mathcal{N}\!\bigl(\mathbf{x}_t;\;
       \sqrt{\bar\alpha_t}\,\mathbf{x}_0,\;
       (1-\bar\alpha_t)\,\mathbf{I}\bigr), \\
  &\quad t \in \{0, 1, \ldots, T-1\},
\end{aligned}
\label{eq:ddpm-forward}
\end{equation}
with $\bar\alpha_t = \prod_{s=0}^{t}(1-\beta_s)$ and a fixed linear
$\beta$-schedule \citep{ho2020denoising}.
We use the pretrained \texttt{google/ddpm-cifar10-32} checkpoint, whose
$\bar\alpha$-values satisfy $\bar\alpha_0 \approx 0.9999$ (almost clean)
and $\bar\alpha_{T-1} \approx 0.0047$ (almost pure noise).
All subsequent computation reads $\bar\alpha_t$ \emph{from the model's
own buffer}, ensuring perfect alignment with the values used during
training.
 
\paragraph{Sampling schedule.}
A DDPM \emph{sampling schedule} is an integer-valued sequence
\begin{equation}
\begin{aligned}
\boldsymbol\tau &= (\tau_0, \tau_1, \ldots, \tau_N), \\
&\ \tau_i \in \{0, 1, \ldots, T-1\}, \\
&\ \tau_0 = T-1,\ \ \tau_N = 0,
\end{aligned}
\label{eq:ddpm-schedule}
\end{equation}
specifying which $N$ timesteps the sampler queries the denoiser at.
The total number of denoiser evaluations (NFE) is exactly $N$.
 
\paragraph{Generalised DDIM update.}
At each step, the predicted clean image is computed and clipped:
\begin{equation}
\begin{aligned}
&\xhatz(\mathbf{x}_{\tau_i}, \tau_i) \\
&\;= \mathrm{clip}_{[-1,1]}\!\!\left(\!
   \frac{\mathbf{x}_{\tau_i} - \sqrt{1-\bar\alpha_{\tau_i}}\,
         \bm\varepsilon_\theta(\mathbf{x}_{\tau_i}, \tau_i)}
        {\sqrt{\bar\alpha_{\tau_i}}}
   \!\right)\!.
\end{aligned}
\label{eq:x0hat}
\end{equation}
The noise prediction is then recomputed from the (possibly clipped)
$\xhatz$,
\begin{equation}
\tilde{\bm\varepsilon}
  = \frac{\mathbf{x}_{\tau_i} - \sqrt{\bar\alpha_{\tau_i}}\,\xhatz}
         {\sqrt{1-\bar\alpha_{\tau_i}}},
\label{eq:eps-recompute}
\end{equation}
and the next state is
\begin{equation}
\boxed{
\begin{aligned}
\mathbf{x}_{\tau_{i+1}} = {} & \sqrt{\bar\alpha_{\tau_{i+1}}}\,\xhatz \\
  & + \sqrt{1 - \bar\alpha_{\tau_{i+1}} - \sigma_i^2}\,\tilde{\bm\varepsilon} \\
  & + \sigma_i\,\mathbf{z}, \\
  & \mathbf{z} \sim \mathcal{N}(\mathbf{0},\mathbf{I}),
\end{aligned}}
\label{eq:ddim-generalised}
\end{equation}
where the per-step variance $\sigma_i^2$ controls stochasticity:
\begin{equation}
\sigma_i^2 =
\begin{cases}
\eta^2 \dfrac{1-\bar\alpha_{\tau_{i+1}}}{1-\bar\alpha_{\tau_i}}\!
  \left(\!1 - \dfrac{\bar\alpha_{\tau_i}}{\bar\alpha_{\tau_{i+1}}}\!\right) \\[6pt]
\quad \text{if } \bar\alpha_{\tau_{i+1}} > \bar\alpha_{\tau_i}\
   \text{(denoising step)}, \\[3pt]
0 \\[2pt]
\quad \text{if } \bar\alpha_{\tau_{i+1}} \leq \bar\alpha_{\tau_i}\
   \text{(reheat step)}.
\end{cases}
\label{eq:sigma-eta}
\end{equation}
The interpolation parameter $\eta \in [0,1]$ recovers deterministic
DDIM at $\eta=0$ \citep{song2021ddim} and stochastic DDPM-like sampling
at $\eta=1$ \citep{ho2020denoising}.
We disable stochasticity on reheat steps because the standard variance
formula in \eqref{eq:sigma-eta} becomes undefined when
$\bar\alpha_{\tau_{i+1}} \leq \bar\alpha_{\tau_i}$.
 
\begin{remark}[Schedule alignment]
The values $\bar\alpha_{\tau_i}$ in \eqref{eq:x0hat}--\eqref{eq:ddim-generalised}
are read directly from the model's training buffer.
Constructing an independent $\bar\alpha$-array (e.g.\ via a cosine
schedule \citep{nichol2021improved}) when the model was trained with
linear $\beta$ silently breaks the DDIM correspondence and produces
visibly degraded samples; the aligned-buffer approach is essential to
the validity of all non-monotonic experiments below.
\end{remark}
 
\subsection{Preliminaries: EDM Sampling on Continuous Noise Levels}
\label{subsec:prelim-edm}
 
\paragraph{Karras $\rho$-schedule.}
EDM \citep{karras2022edm} parameterises the sampler by a continuous
noise level $\sigma > 0$.
The default monotonic schedule, with $N$ denoiser evaluations, is
\begin{equation}
\begin{aligned}
\sigma_i &= \left(\sigma_{\max}^{1/\rho}
  + \frac{i}{N-1}\bigl(\sigma_{\min}^{1/\rho} - \sigma_{\max}^{1/\rho}\bigr)
  \right)^{\!\rho}, \\
&\quad i = 0, 1, \ldots, N-1,
\end{aligned}
\label{eq:edm-rho}
\end{equation}
followed by a final $\sigma_N = 0$ noise-free step.
We use the EDM defaults: $\sigma_{\min} = 0.002$, $\sigma_{\max} = 80$,
$\rho = 7$.
The schedule is strictly decreasing in $i$.
 
\paragraph{Sampler.}
Each step of the (Euler) sampler integrates the probability-flow ODE,
\begin{equation}
\begin{aligned}
\mathbf{x}_{i+1} = {} & \mathbf{x}_i \\
  & + (\sigma_{i+1} - \sigma_i)\,
      \frac{\mathbf{x}_i - D_\theta(\mathbf{x}_i, \sigma_i)}{\sigma_i},
\end{aligned}
\label{eq:edm-euler}
\end{equation}
where $D_\theta$ is the EDM denoiser with built-in preconditioning.
Heun's second-order corrector is used optionally
(see \cref{app:heun}); we report Euler results in the main paper.
A reheat step is one with $\sigma_{i+1} > \sigma_i$, making the step
size $\sigma_{i+1}-\sigma_i$ positive—i.e.\ re-noising the current
state.
This is mathematically valid for any $\sigma_{i+1} \in (0, \sigma_{\max}]$,
exactly analogous to \cref{prop:validity} below.
 
\subsection{Preliminaries: Flow-Matching Sampling on Continuous Time}
\label{subsec:prelim-fm}
 
\paragraph{OT-CFM.}
We train a flow-matching model under the optimal-transport conditional
flow-matching (OT-CFM) objective
\citep{lipman2023flow,liu2023rectified}.
The interpolation between noise
$\bm\varepsilon \sim \mathcal{N}(\mathbf{0},\mathbf{I})$ and clean data
$\mathbf{x}_0 \sim p_{\mathrm{data}}$ is linear,
\begin{equation}
\mathbf{x}_t \;=\; (1-t)\,\bm\varepsilon \;+\; t\,\mathbf{x}_0,
\quad t \in [0, 1],
\label{eq:ot-cfm-interp}
\end{equation}
and the velocity target is
$\mathbf{v}^\star = \mathbf{x}_0 - \bm\varepsilon$, with the network
$\mathbf{v}_\theta(\mathbf{x}_t, t)$ trained on
$\mathcal{L}(\theta) = \mathbb{E}\,\|\mathbf{v}_\theta - \mathbf{v}^\star\|_2^2$.
At $t=0$ the state is pure noise; at $t=1$ it equals data.
 
\paragraph{Sampling schedule and update.}
A flow-matching schedule is a sequence of times
\begin{equation}
\begin{aligned}
\mathbf{t} &= (t_0, t_1, \ldots, t_N), \\
&\ t_i \in [t_{\min}, t_{\max}], \\
&\ t_0 = t_{\min},\ \ t_N = t_{\max},
\end{aligned}
\label{eq:fm-schedule}
\end{equation}
with $t_{\min} = 0.001$, $t_{\max} = 0.999$ to avoid endpoint
singularities.
The sampler performs Euler integration:
\begin{equation}
\mathbf{x}_{t_{i+1}}
  = \mathbf{x}_{t_i}
  + (t_{i+1} - t_i)\,\mathbf{v}_\theta(\mathbf{x}_{t_i}, t_i).
\label{eq:fm-euler}
\end{equation}
A reheat step is one with $t_{i+1} < t_i$, producing a negative
$\Delta t$.
The Euler update remains well-defined for either sign of $\Delta t$.
 
\subsection{Reheat Steps and the Unified Validity Statement}
\label{subsec:reheat-formal}
 
Each of the three samplers admits a natural notion of \emph{progress}
toward clean data: $\bar\alpha_t$ increases along the standard sampling
direction for DDPM, $\sigma$ decreases for EDM, and $t$ increases for
flow matching.
We unify them via a generic noise level $\hat\sigma$ that is strictly
\emph{decreasing} along the standard monotonic trajectory:
\begin{align}
\text{DDPM:}\quad & \hat\sigma_{\tau_i} = \sqrt{1-\bar\alpha_{\tau_i}}, \\
\text{EDM:}\quad  & \hat\sigma_i = \sigma_i, \\
\text{FM:}\quad   & \hat\sigma_{t_i} = 1 - t_i.
\end{align}
Under this convention, every sampler shares one definitional property
of the standard monotonic baseline:
$\hat\sigma_0 > \hat\sigma_1 > \cdots > \hat\sigma_N$.
 
\begin{definition}[Reheat step]
\label{def:reheat-step}
The transition from index $i$ to index $i+1$ is a \textbf{reheat step}
if $\hat\sigma_{i+1} > \hat\sigma_i$; otherwise it is a
\textbf{denoising step}.
We denote the set of reheat-step indices in a schedule by
$\mathcal{R} = \{i : \hat\sigma_{i+1} > \hat\sigma_i\}$.
A schedule is \textbf{monotonic} iff $\mathcal{R} = \emptyset$, and
\textbf{reheating} iff $\mathcal{R} \neq \emptyset$.
\end{definition}
 
\begin{proposition}[Validity of non-monotonic sampling]
\label{prop:validity}
The sampler updates
\eqref{eq:ddim-generalised} (DDPM/DDIM),
\eqref{eq:edm-euler} (EDM Euler),
and \eqref{eq:fm-euler} (flow matching) are all well-defined for an
arbitrary next-index value, including those producing reheat steps.
In particular, no retraining of the denoiser
$\bm\varepsilon_\theta$ / $D_\theta$ / $\mathbf{v}_\theta$ is required.
\end{proposition}
 
\begin{proof}
Each update is a closed-form expression in
$(\mathbf{x}_i, \hat\sigma_i, \hat\sigma_{i+1}, \theta)$ whose
coefficients are continuous functions of $\hat\sigma_{i+1}$ on the open
interval $(\hat\sigma_{\min}, \hat\sigma_{\max})$.
The denoiser is queried only at the \emph{current} index $i$;
its inputs are unchanged by the choice of $\hat\sigma_{i+1}$.
Hence the update remains computable for any $\hat\sigma_{i+1}$ in the
interior of the noise range, regardless of its sign relative to
$\hat\sigma_i$.
For the DDPM update specifically, when $\hat\sigma_{i+1} > \hat\sigma_i$
the variance term $\sigma_i^2$ in \eqref{eq:sigma-eta} is set to zero
to preserve well-definedness; the resulting reheat step is the natural
deterministic re-noising of $\xhatz$ at the higher noise level
(\cref{remark:reheat-geom}).
\end{proof}
 
\begin{remark}[Geometry of a deterministic reheat step]
\label{remark:reheat-geom}
Substituting $\sigma_i = 0$ (per \eqref{eq:sigma-eta}) into
\eqref{eq:ddim-generalised} on a reheat step yields
\begin{equation}
\begin{aligned}
\mathbf{x}_{\tau_{i+1}} = {} & \sqrt{\bar\alpha_{\tau_{i+1}}}\,\xhatz \\
  & + \sqrt{1-\bar\alpha_{\tau_{i+1}}}\,\tilde{\bm\varepsilon},
\end{aligned}
\label{eq:reheat-geom}
\end{equation}
which is exactly the forward-process form \eqref{eq:ddpm-forward} with
$\mathbf{x}_0$ replaced by the model's prediction $\xhatz$ and noise
direction $\bm\varepsilon$ replaced by the reconstructed
$\tilde{\bm\varepsilon}$.
A reheat step thus \emph{re-corrupts the model's current best estimate
of the clean image} along the direction of the noise it just identified.
The state after the reheat has higher noise than $\mathbf{x}_{\tau_i}$
by construction; the sampler must then re-denoise from
$\hat\sigma_{\tau_{i+1}}$, traversing already-visited noise levels.
The same geometric picture holds for the EDM and FM updates after the
identifications $\hat\sigma_{\mathrm{EDM}} = \sigma$ and
$\hat\sigma_{\mathrm{FM}} = 1-t$.
\end{remark}
 
\begin{definition}[Reheating overhead]
\label{def:overhead}
For a schedule of $N$ steps, the \textbf{reheating overhead} is
\begin{equation}
\mathcal{O}
  \coloneqq
  \frac{1}{\hat\sigma_0 - \hat\sigma_N}
  \sum_{i \in \mathcal{R}}
  \bigl(\hat\sigma_{i+1} - \hat\sigma_i\bigr)
  \;\geq\; 0,
\label{eq:overhead-def}
\end{equation}
with $\mathcal{O} = 0$ iff the schedule is monotonic.
$\mathcal{O}$ measures the cumulative noise added by reheat steps as a
fraction of the total noise range covered by the schedule, so a
schedule with overhead $\mathcal{O}$ effectively spends $\mathcal{O}$
of its NFE budget re-denoising already-visited noise levels.
\end{definition}
 
\subsection{Schedule Families}
\label{subsec:families}
 
We describe each schedule family in DDPM integer-timestep form;
the corresponding EDM ($\sigma$-space) and flow-matching ($t$-space)
constructions are listed in \cref{app:schedules}.
All four families share a common \emph{base monotonic schedule},
\begin{equation}
\begin{aligned}
\tau^{\mathrm{m}}_i &= \mathrm{round}\!\left((T-1) \cdot \frac{N - i}{N}\right), \\
&\quad i = 0, 1, \ldots, N,
\end{aligned}
\label{eq:base-mono}
\end{equation}
with the endpoints $\tau^{\mathrm{m}}_0 = T-1$ and $\tau^{\mathrm{m}}_N = 0$
enforced exactly.
\Cref{fig:schedules} illustrates all four families schematically.
 
\paragraph{(F1) Monotonic baseline.}
Returns $\boldsymbol\tau^{\mathrm{mono}} = (\tau^{\mathrm{m}}_0, \ldots,
\tau^{\mathrm{m}}_N)$ verbatim from \eqref{eq:base-mono}.
The reheat set is empty: $\mathcal{R} = \emptyset$, $\mathcal{O} = 0$.
 
\paragraph{(F2) Single Reheat (SR).}
The minimal non-monotonic deviation: one backward jump in timestep at
position $r = \mathrm{clip}(\lfloor t_{\mathrm{reheat}} \cdot N \rfloor,\, 2,\, N-3)$.
The reheat magnitude is parameterised \emph{multiplicatively} on the
current timestep,
\begin{equation}
\begin{aligned}
\Delta\tau^{\mathrm{SR}} &= \max\!\bigl(\lfloor \tau^{\mathrm{m}}_r \cdot \delta \rfloor,\; 1\bigr), \\
\tau^{\mathrm{SR}}_{\mathrm{reheat}} &= \min\!\bigl(\tau^{\mathrm{m}}_r + \Delta\tau^{\mathrm{SR}},\; T-1\bigr).
\end{aligned}
\label{eq:sr-update}
\end{equation}
The full schedule is constructed by inserting the reheat timestep
$\tau^{\mathrm{SR}}_{\mathrm{reheat}}$ after the $r$-step prefix of
the monotonic schedule and resuming with a linear descent to zero:
\begin{equation}
\begin{aligned}
&\boldsymbol\tau^{\mathrm{SR}}(t_{\mathrm{reheat}}, \delta) = \bigl(\\
&\quad \underbrace{\tau^{\mathrm{m}}_0, \ldots, \tau^{\mathrm{m}}_{r-1}}_{r\ \text{prefix steps}},\; \\
&\quad \underbrace{\tau^{\mathrm{SR}}_{\mathrm{reheat}}}_{\text{reheat}},\; \\
&\quad \underbrace{\mathrm{linspace}\!\bigl(\tau^{\mathrm{SR}}_{\mathrm{reheat}}, 0, N-r\bigr)}_{\text{linear descent}}\,\bigr).
\end{aligned}
\label{eq:sr-schedule}
\end{equation}
\textit{Defaults:}\ $t_{\mathrm{reheat}} = 0.4$, $\delta = 0.15$.
\textit{Ablation grid:}\ $t_{\mathrm{reheat}} \in
\{0.2, 0.3, 0.4, 0.5, 0.6, 0.7, 0.8\}$,
$\delta \in \{0.05, 0.10, 0.15, 0.20, 0.30, 0.50\}$
($42$ configurations; \cref{sec:exp-ablation}).
 
\paragraph{(F3) Sawtooth (ST).}
Periodic backward jumps every $P$ steps, each of \emph{multiplicative}
magnitude $\delta_{\mathrm{ST}}$ on the current timestep.
For each $i \in \{P, 2P, 3P, \ldots\}$ with $i < N-2$ and
$\tau^{\mathrm{m}}_i \geq 5$,
\begin{equation}
\begin{aligned}
\tau^{\mathrm{ST}}_i = \min\!\bigl(&\tau^{\mathrm{m}}_i + \max(\lfloor \tau^{\mathrm{m}}_i \cdot \delta_{\mathrm{ST}} \rfloor,\, 1),\, \\
&T-1\bigr),
\end{aligned}
\label{eq:sawtooth-update}
\end{equation}
and otherwise $\tau^{\mathrm{ST}}_i = \tau^{\mathrm{m}}_i$.
The trailing condition $\tau^{\mathrm{m}}_i \geq 5$ suppresses
near-terminal reheats where $\Delta\tau$ would round to zero anyway.
\textit{Note:}\ because the multiplier scales the current timestep, the
absolute jump $\Delta\tau$ \emph{decreases} as the trajectory
approaches $\tau = 0$—the schedule is geometrically tapered, not flat.
\textit{Defaults:}\ $P = 25$, $\delta_{\mathrm{ST}} = 0.08$.
 
\paragraph{(F4) Damped Oscillation (DO).}
A continuous damped sinusoid is added to the linear timestep ramp.
With normalised step coordinate $s_i = i/N \in [0, 1]$:
\begin{equation}
\boxed{
\begin{aligned}
\tau^{\mathrm{DO}}_i = \mathrm{clip}_{[0,T-1]}\!\Bigl(\;
  & \underbrace{(1-s_i)(T-1)}_{\text{linear descent}} \\
  +\; & \underbrace{A\,(T-1)\,e^{-\gamma\,s_i}\,\sin(2\pi f s_i)}_{\text{damped oscillation}}
  \;\Bigr),
\end{aligned}}
\label{eq:do-schedule}
\end{equation}
followed by integer rounding and the endpoint pinning
$\tau^{\mathrm{DO}}_0 = T-1$, $\tau^{\mathrm{DO}}_N = 0$.
The amplitude $A$ controls the depth of oscillation,
$\gamma$ controls the damping rate (so that oscillations vanish toward
$s = 1$, where fine details are determined), and $f$ controls the
number of full sinusoidal cycles.
\textit{Defaults:}\ $A = 0.2$, $\gamma = 2.5$, $f = 4$.
At these defaults the schedule contains $\sim\!2$ oscillation humps
of $6$--$8$ reheat sub-steps each, concentrated where global image
structure is determined (mid-trajectory).
 
\paragraph{(F5) Adaptive Reheat (AR).}
The first four families are deterministic functions of their
hyperparameters.
The Adaptive Reheat triggers reheat steps \emph{online}, based on a
calibrated stability criterion on the model's clean-image prediction.
At each step $i \geq 1$ the sampler computes the root-mean-square
change in $\xhatz$ between consecutive steps,
\begin{equation}
\Delta^{(i)}_{\mathrm{AR}}
  = \sqrt{\frac{1}{D}\bigl\|\, \xhatz^{(i)} - \xhatz^{(i-1)} \bigr\|_2^2},
\label{eq:ar-criterion}
\end{equation}
where $D$ is the data dimensionality (e.g.\ $D = 3 \cdot 32 \cdot 32 = 3072$
for CIFAR-10).
If $\Delta^{(i)}_{\mathrm{AR}} > \tau$ and the reheat budget is not
exhausted, the sampler inserts a backward jump of
$\Delta\tau_{\mathrm{AR}}$ integer timesteps after step $i$, so the
next index $i+1$ is at timestep
$\min(\tau_i + \Delta\tau_{\mathrm{AR}},\, T-1)$, replacing the would-be
$\tau^{\mathrm{m}}_{i+1}$.
\textit{Defaults:}\ $\Delta\tau_{\mathrm{AR}} = 50$, maximum reheats
$= 15$.
 
\paragraph{Calibration of the threshold $\tau$.}
The threshold is data-driven, not hand-tuned.
Before sampling, we run $K_{\mathrm{cal}} = 100$ \emph{monotonic}
trajectories of $N=50$ steps each and collect all observed
$\Delta^{(i)}_{\mathrm{AR}}$ values; $\tau$ is set to the
$80$-th percentile of this empirical distribution:
\begin{equation}
\begin{aligned}
\tau = \mathrm{Quantile}_{0.80}\!\Bigl(\bigl\{
  & \Delta^{(i)}_{\mathrm{AR}} \,:\, \\
  & k = 1,\ldots, K_{\mathrm{cal}}; \\
  & i = 1,\ldots, N-1\bigr\}\Bigr).
\end{aligned}
\label{eq:tau-calibration}
\end{equation}
Reheat thus fires only on the top $20\%$ most volatile transitions
along a typical monotonic run.
Calibration adds $K_{\mathrm{cal}} \cdot N$ NFE of one-shot overhead
($5{,}000$ in our setup); this overhead is amortised over all
subsequent batches and is included in the wall-clock figures
reported in \cref{tab:ddpm-results}.
 
\paragraph{NFE accounting.}
For the deterministic families F1--F4, the schedule has length $N+1$
and produces exactly $N$ denoiser evaluations.
A schedule with $|\mathcal{R}|$ reheat steps therefore performs
$N - |\mathcal{R}|$ \emph{net} denoising steps and uses
$|\mathcal{R}|$ steps re-denoising already-visited noise levels.
This is the precise statistical sense in which reheating \emph{trades}
denoising NFE for non-monotonic exploration; our experiments hold $N$
fixed and vary the trade.
The Adaptive Reheat is the one exception: its inserted reheat steps
\emph{add} to the budget up to a cap of $15$, and we report both
fixed-NFE and time-matched comparisons in \cref{sec:exp-ddpm}.
 
\subsection{Connection to DDPM Stochasticity}
\label{subsec:ddpm-connection}
 
Setting $\eta = 1$ in \eqref{eq:ddim-generalised} recovers the
stochastic DDPM sampler \citep{ho2020denoising}, which adds fresh
Gaussian noise $\sigma_i\,\mathbf{z}$ at every denoising step.
Because every step still satisfies
$\bar\alpha_{\tau_{i+1}} > \bar\alpha_{\tau_i}$, this is \emph{not} a
reheat step under \cref{def:reheat-step}: the net trajectory is
monotonic in $\hat\sigma$.
The fresh-noise term, however, has an interpretation that connects
DDPM stochasticity to our explicit reheating schedules.
 
\paragraph{Decomposition of the stochastic term.}
Let $\hat{\bm\varepsilon} = \tilde{\bm\varepsilon}/\|\tilde{\bm\varepsilon}\|_2$
denote the unit vector in the model's predicted noise direction.
Decompose the fresh noise $\mathbf{z}$ into components parallel and
orthogonal to $\hat{\bm\varepsilon}$:
\begin{equation}
\begin{aligned}
\mathbf{z} = {} &
  \underbrace{\langle \mathbf{z},\,\hat{\bm\varepsilon}\rangle\,\hat{\bm\varepsilon}}_{\mathbf{z}_\parallel\
     \text{(implicit reheat)}} \\
  + & \underbrace{\bigl(\mathbf{z}
       - \langle \mathbf{z},\,\hat{\bm\varepsilon}\rangle\,\hat{\bm\varepsilon}\bigr)}_{\mathbf{z}_\perp\
     \text{(trajectory diversification)}}.
\end{aligned}
\label{eq:ddpm-decomp}
\end{equation}
The parallel component $\mathbf{z}_\parallel$ is collinear with the
predicted noise direction; adding $\sigma_i\,\mathbf{z}_\parallel$ to
the deterministic update partially counteracts the denoising progress
of the
$\sqrt{1-\bar\alpha_{\tau_{i+1}}-\sigma_i^2}\,\tilde{\bm\varepsilon}$
term—an implicit, infinitesimal, random-magnitude reheat.
The orthogonal component $\mathbf{z}_\perp$ moves the sample
\emph{across} the level set of the score function, effecting trajectory
diversification on the data manifold's null space, a mechanism
unavailable to deterministic schedules.
 
\paragraph{Implications.}
Our deterministic reheating schedules (F2--F4 with $\eta=0$) reside
entirely in the $\mathbf{z}_\parallel$ direction: they re-noise along
the predicted-noise direction with no orthogonal exploration.
Our stochastic reheating variant (used in \cref{tab:ddpm-results} and
\cref{fig:pareto}) sets $\eta = 0.5$ at \emph{denoising} steps within
an otherwise non-monotonic schedule, decoupling the two effects:
non-monotonic exploration via the deterministic reheat step, and
orthogonal diversification via $\eta > 0$ on adjacent denoising steps.
We empirically confirm in \cref{sec:exp-ddpm} that the orthogonal
component is what drives DDPM ($\eta=1$) past DDIM at NFE $=200$, and
that pure deterministic reheating cannot reproduce this benefit at any
tested NFE.
 
\begin{remark}[What this paper tests, and what it does not]
The question is not whether stochasticity helps in diffusion sampling
(it does, as documented since DDPM and analysed thoroughly in EDM
\citep{karras2022edm}).
The question is whether \emph{structured non-monotonicity in the noise
schedule} helps, holding sampler stochasticity decoupled from schedule
design.
Our experiments in \cref{sec:experiments} answer this question
negatively across three model families, with full cross-model results for Damped Oscillation and DDPM-level screening for Sawtooth.
\end{remark}

\section{Experiments}
\label{sec:experiments}

We evaluate the four schedule families of \cref{subsec:families} against
the monotonic baseline on three model families—DDPM, EDM, and Flow
Matching—across NFE budgets ranging from $10$ to $200$ on CIFAR-10.
The experiments are organised around four questions, each answered by a
single subsection: \emph{does reheating help any single model?}
(\cref{sec:exp-ddpm}), \emph{does the answer depend on the architecture?}
(\cref{sec:exp-cross}), \emph{is the negative result robust to
hyperparameter choice?} (\cref{sec:exp-ablation}), and
\emph{does reheating offer any wall-clock advantage?}
(\cref{sec:exp-pareto}).

\subsection{Setup}
\label{sec:exp-setup}

\paragraph{Models and checkpoints.}
For DDPM we use the pretrained \texttt{google/ddpm-cifar10-32}
checkpoint (35.7M parameters, trained for $T=1000$ timesteps).
For EDM we use the pretrained \texttt{edm-cifar10-32x32-uncond-vp}
checkpoint of \citet{karras2022edm} (55.7M parameters).
For Flow Matching we train a small UNet ($\sim$5M parameters) under the
OT-CFM objective \citep{lipman2023flow} for 50K iterations on
CIFAR-10; the lightweight architecture is adequate to test schedule
sensitivity (we revisit its absolute FID in \cref{sec:limitations}).
All three models are kept frozen at their training checkpoints for
every experiment in this section—the only thing that varies between
runs is the noise schedule.

\paragraph{Evaluation protocol.}
Sample quality is reported as Fréchet Inception Distance (FID;
\citealp{heusel2017fid}) computed against the full $50{,}000$-image
CIFAR-10 training set using the standard Inception-V3 feature
extractor.
We use $50{,}000$ generated samples for the DDPM and Flow Matching
experiments (\cref{tab:ddpm-results,tab:cross-model}), $25{,}000$ for
the EDM experiments and the Pareto comparison
(\cref{tab:cross-model,fig:pareto}), and $10{,}000$ for the 42-cell
ablation grid (\cref{tab:ablation,fig:ablation}).
Sample sizes are chosen to balance statistical power against the
per-run cost; all comparisons within a single table use the same
sample size.
We seed every batch deterministically, so across-row differences
within a single table are attributable to schedule choice alone, not
to noise in the FID estimator.

\paragraph{Schedule defaults.}
Unless otherwise stated, we use the defaults of \cref{subsec:families}:
$t_\text{reheat}=0.4$, $\delta=0.15$ for Single Reheat;
$P=25$, $\delta_\text{ST}=0.08$ for Sawtooth;
$A=0.2$, $\gamma=2.5$, $f=4$ for Damped Oscillation;
$\Delta\tau_\text{AR}=50$, max-reheats~$=15$, $\tau$ calibrated at the
$80$\textsuperscript{th} percentile of $100$ trajectories of $N=50$ for
Adaptive Reheat.
The full hyperparameter grid for the ablation is
$t_\text{reheat} \in \{0.2,0.3,0.4,0.5,0.6,0.7,0.8\}$ and
$\delta \in \{0.05,0.10,0.15,0.20,0.30,0.50\}$ (42 configurations).

\subsection{DDPM: All Five Reheating Variants Underperform DDIM}
\label{sec:exp-ddpm}

\Cref{tab:ddpm-results} reports FID-50K on CIFAR-10 for the monotonic
DDIM baseline and five reheating variants, at NFE budgets $\{10, 25,
50, 100\}$.

\begin{table*}[t]
\centering
\caption{FID-50K on CIFAR-10 (\texttt{google/ddpm-cifar10-32}, $n=50\,000$ samples).
Bold denotes the best result per NFE column.
Parentheses show the absolute FID penalty relative to the DDIM baseline.
Adaptive Reheat incurs an additional wall-clock overhead that varies with NFE: $1.89\times$ DDIM at NFE = $10$ shrinking to $1.15\times$ at NFE = $100$ (see \cref{sec:exp-pareto} for numerical detail)}
\label{tab:ddpm-results}
\setlength{\tabcolsep}{6pt}
\begin{tabular}{lcccc}
\toprule
\textbf{Method} & \textbf{NFE\,=\,10} & \textbf{NFE\,=\,25} & \textbf{NFE\,=\,50} & \textbf{NFE\,=\,100} \\
\midrule
DDIM (monotonic)            & \textbf{38.97}           & \textbf{21.88}           & \textbf{17.33}           & \textbf{15.28}           \\
DDPM ($\eta=1$)             & 59.13 {\small(+20.16)}   & 29.40 {\small(+7.52)}    & 20.54 {\small(+3.21)}    & 15.80 {\small(+0.52)}    \\
Single reheat               & 43.25 {\small(+4.28)}    & 23.46 {\small(+1.58)}    & 18.22 {\small(+0.89)}    & 16.00 {\small(+0.72)}    \\
Damped oscillation          & 72.31 {\small(+33.34)}   & 29.18 {\small(+7.30)}    & 20.28 {\small(+2.96)}    & 16.59 {\small(+1.31)}    \\
Adaptive reheat$^\dagger$   & 76.62 {\small(+37.65)}   & 36.26 {\small(+14.38)}   & 23.13 {\small(+5.80)}    & 16.86 {\small(+1.58)}    \\
Reheat + stochastic         & 45.89 {\small(+6.92)}    & 24.17 {\small(+2.29)}    & 18.23 {\small(+0.90)}    & 15.52 {\small(+0.24)}    \\
\bottomrule
\multicolumn{5}{l}{{\small $^\dagger$ Adaptive Reheat uses online overhead; wall-clock time ranges from \(1.89\times\) DDIM at NFE \(=10\) to \(1.15\times\) at NFE \(=100\).}}
\end{tabular}
\end{table*}

The result is unambiguous: DDIM (monotonic, $\eta=0$) achieves the
lowest FID at every NFE budget tested.
The penalty pattern is highly informative.
At low NFE (10 steps), reheating is catastrophic: the Damped
Oscillation schedule loses $+33.34$ FID points relative to DDIM
($72.31$ vs.\ $38.97$), and even the minimal Single Reheat schedule
loses $+4.28$.
The penalty shrinks with NFE but never closes: at NFE = $100$, Damped
Oscillation still costs $+1.31$ FID, Single Reheat costs $+0.72$, and
Reheat~+~Stochastic ($\eta = 0.5$ on the denoising steps within the
Single Reheat schedule) is the least-bad reheating variant at
$+0.24$—a small but persistent gap.

The Adaptive Reheat is the worst-performing variant despite its online
data-dependent design (\cref{eq:ar-criterion,eq:tau-calibration}):
$+37.65$ FID at NFE = 10 and $+1.58$ at NFE = 100.
Crucially, this poor sample quality coexists with an NFE-dependent wall-clock overhead relative to DDIM, ranging from \(1.89\times\) at NFE \(=10\) to \(1.15\times\) at NFE \(=100\), because the threshold-checking step requires retaining \(\hat{x}^{(i-1)}_0\) across iterations and the inserted reheat steps add to the total NFE budget rather than replacing denoising steps.

The fully-stochastic DDPM ($\eta = 1$) variant also underperforms DDIM
at every tested budget (e.g.\ $+0.52$ FID at NFE = 100).
This matches the well-documented behaviour at moderate NFE
\citep{karras2022edm}: in this regime, DDIM's deterministic sampler
sits on the Pareto frontier.
The complementary observation—that DDPM eventually surpasses DDIM
at higher NFE—is the subject of \cref{sec:exp-pareto}.

\paragraph{Sawtooth.}
We screened the Sawtooth schedule (\cref{eq:sawtooth-update}) at NFE = $25$
on DDPM and observed a penalty of $+5.8$ FID at default
$\delta_\text{ST} = 0.08$, intermediate between Single Reheat and Damped
Oscillation but qualitatively identical (positive, increasing in
$\delta_\text{ST}$).
Because Sawtooth offers no distinct mechanism beyond what Single Reheat
and Damped Oscillation already probe, we omit it from the main
cross-model and Pareto experiments and report the screening result here
for completeness.
\subsection{Cross-Model Comparison: The Penalty Magnitude is
Architecture-Dependent}
\label{sec:exp-cross}

\Cref{tab:cross-model} reports the same Damped Oscillation schedule
applied to all three model families.
\Cref{fig:penalty-convergence} visualises the FID penalty
($\Delta\text{FID} = \text{FID}_\text{damped} - \text{FID}_\text{monotonic}$) on a
logarithmic axis.

\begin{table*}[t]
\centering
\caption{Cross-model FID comparison for the damped oscillation schedule.
\emph{Monotonic} = DDIM baseline; \emph{Damp.\ osc.} = damped oscillation variant;
\emph{Penalty} = $\Delta$FID (positive means worse).
DDPM: $n=50\,000$ samples, CIFAR-10 (Exp.~2).
EDM: $n=25\,000$ samples, CIFAR-10 (Exp.~3).
Flow matching: $n=50\,000$ samples, CIFAR-10 (Exp.~5).
The EDM penalty at NFE\,=\,100 is negative ($-0.003$), indicating
statistical indistinguishability from zero at 25K samples.}
\label{tab:cross-model}
\setlength{\tabcolsep}{4.5pt}
\begin{tabular}{l *{3}{r} *{3}{r} *{3}{r}}
\toprule
& \multicolumn{3}{c}{\textbf{DDPM}} 
& \multicolumn{3}{c}{\textbf{EDM (Euler)}} 
& \multicolumn{3}{c}{\textbf{Flow matching}} \\
\cmidrule(lr){2-4}\cmidrule(lr){5-7}\cmidrule(lr){8-10}
\textbf{NFE} 
  & Mono. & Damp.\ osc. & Penalty 
  & Mono. & Damp.\ osc. & Penalty 
  & Mono. & Damp.\ osc. & Penalty \\
\midrule
10  & 38.97 & 72.31 & {+33.34} & 19.57 & 20.72 & {+1.15}  & 41.07 & 42.43 & {+1.35} \\
25  & 21.88 & 29.18 & {+7.30}  &  5.99 &  6.06 & {+0.07}  & 39.69 & 40.00 & {+0.31} \\
50  & 17.33 & 20.28 & {+2.96}  &  3.83 &  3.83 & {+0.00}  & 41.16 & 41.21 & {+0.05} \\
100 & 15.28 & 16.59 & {+1.31}  &  3.16 &  3.16 & {$-0.003$} & 42.27 & 42.31 & {+0.04} \\
\bottomrule
\end{tabular}
\end{table*}

\begin{figure*}[t]
  \centering
  \includegraphics[width=\textwidth]{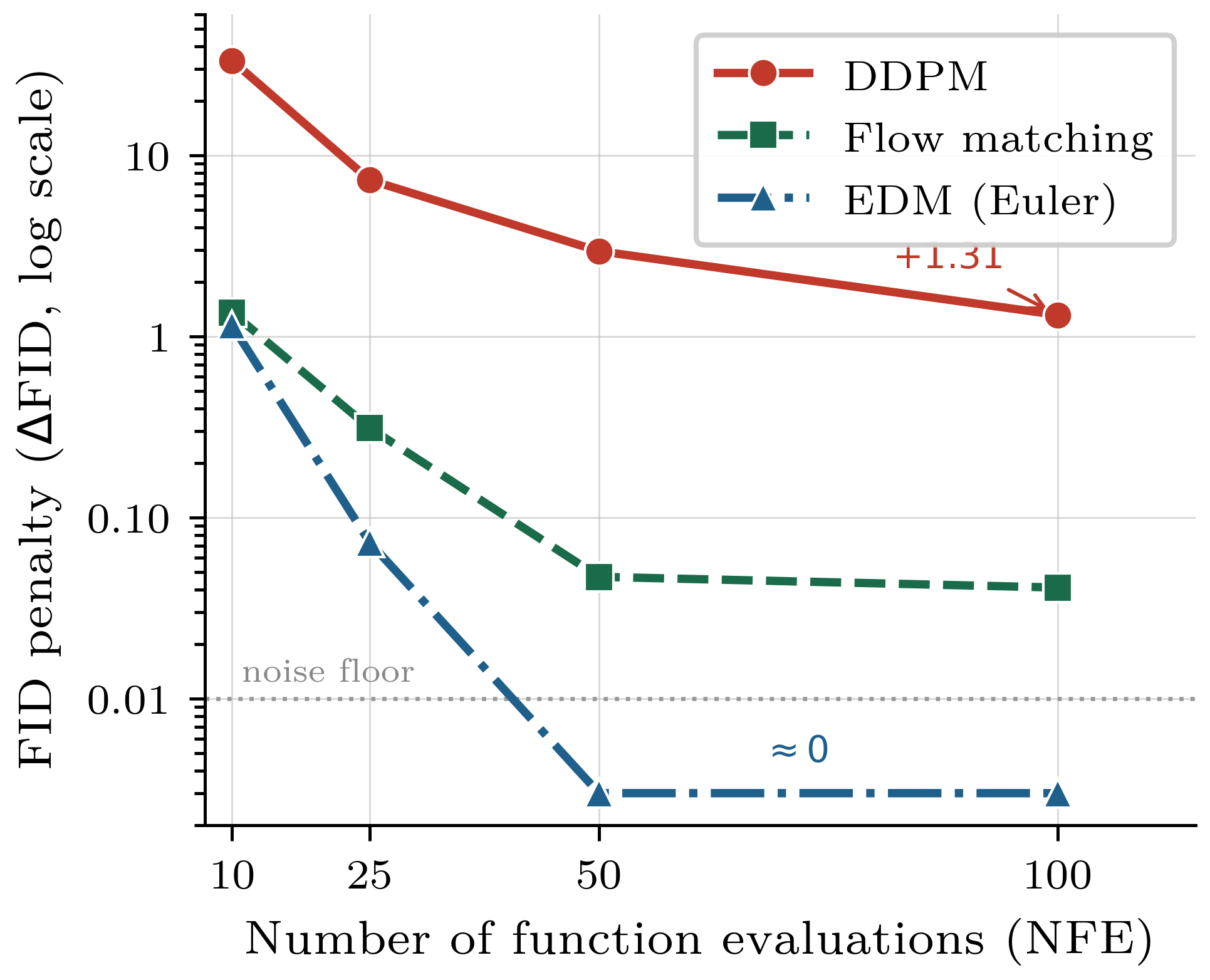}
  \caption{\textbf{FID penalty of the Damped Oscillation schedule
    (vs.\ monotonic baseline) across NFE for three model families.}
    DDPM (red): the penalty starts at $+33.3$ at NFE = $10$ and
    flattens at $+1.31$ at NFE = $100$ without converging to zero.
    Flow Matching (green): an order of magnitude smaller penalty,
    reaching $+0.04$ at NFE = $100$.
    EDM (blue): penalty collapses to $-0.003$ at NFE = $100$,
    statistically indistinguishable from zero at our sample size.
    The horizontal dotted line marks the approximate noise floor for
    FID-25K. Sources: \cref{tab:cross-model}.}
  \label{fig:penalty-convergence}
\end{figure*}

The qualitative behaviour is uniform—no model is helped by reheating
at any tested NFE.
The \emph{quantitative} behaviour, however, varies by nearly three
orders of magnitude.
At NFE = $100$, DDPM's penalty is $+1.314$ FID, Flow Matching's
is $+0.041$ FID, and EDM's is $-0.003$ FID
(within the FID-25K noise floor of $\sim$0.05).
The pattern is monotone in NFE for all three models, but the rate of
convergence differs sharply: EDM's penalty is already at machine zero
by NFE = $50$ ($+0.002$), while DDPM's still sits at $+2.96$ at
the same budget.

We interpret this gap not as a contradiction of our negative result but
as a measurement of \emph{how much each model's denoiser has internalised
the assumption of monotonicity}.
A denoiser whose error is uniformly small at all noise levels and
trajectories should be relatively insensitive to the specific
$\hat\sigma$ history of its inputs.
A denoiser whose error grows or correlates with off-trajectory inputs
will pay a larger penalty when forced to re-process re-noised samples.
This suggests a quantitative diagnostic, the
\emph{Schedule Sensitivity Coefficient} (SSC) of \cref{def:ssc} in
\cref{sec:theory}, which we now report empirically:

\begin{table*}[t]
\centering
\caption{Schedule Sensitivity Coefficient (SSC) estimated at NFE\,=\,100.
SSC = \([\Delta\mathrm{FID}_{\mathrm{DO}}]_+/[\Delta\mathrm{FID}_{\mathrm{SR}}]_+\), with negative FID differences clipped to zero; both penalties are measured relative to the monotonic baseline.
A high SSC indicates persistent compounding errors in the denoiser;
SSC\,$\approx$\,0 indicates robustness to schedule perturbation.
Per Proposition 3, SSC \(>0\) provides evidence that the model has not converged to the Bayes-optimal denoiser at the critical noise level.}
\label{tab:ssc}
\setlength{\tabcolsep}{7pt}
\begin{tabular}{lrrrcc}
\toprule
\textbf{Model} & \textbf{Baseline FID} & \textbf{Damp.\ penalty} & \textbf{Single penalty} & \textbf{SSC} & \textbf{Robustness} \\
\midrule
DDPM                     & 15.28 & +1.314 & +0.717 & 1.83 & Fragile  \\
Flow matching (50K steps)& 42.27 & +0.041 & +0.347 & 0.12 & Intermediate \\
EDM (Euler)              &  3.16 & $-0.003$ & +0.150 & $\approx 0$ & Robust \\
\bottomrule
\end{tabular}
\end{table*}

The SSC values, estimated at NFE \(=100\), span a large dynamic range, from 1.83 for DDPM to a value indistinguishable from zero for EDM, in tight rank correlation with
each model's baseline FID quality.
EDM's SSC is statistically indistinguishable from zero.
We return to the theoretical justification for SSC as a
denoiser-quality probe in \cref{sec:theory}.

\subsection{Ablation: The Negative Result Holds Across All 42
Hyperparameter Configurations}
\label{sec:exp-ablation}

To rule out the concern that our negative result is an artifact of
poorly-chosen hyperparameters, we sweep the full grid of reheat
position $t_\text{reheat}$ and reheat magnitude $\delta$ for the
Single Reheat schedule on DDPM at NFE = $25$.
\Cref{fig:ablation} visualises the resulting $\Delta\text{FID}$ surface;
\cref{tab:ablation} reports per-$\delta$ summary statistics.

\begin{figure*}[t]
  \centering
  \includegraphics[width=\textwidth]{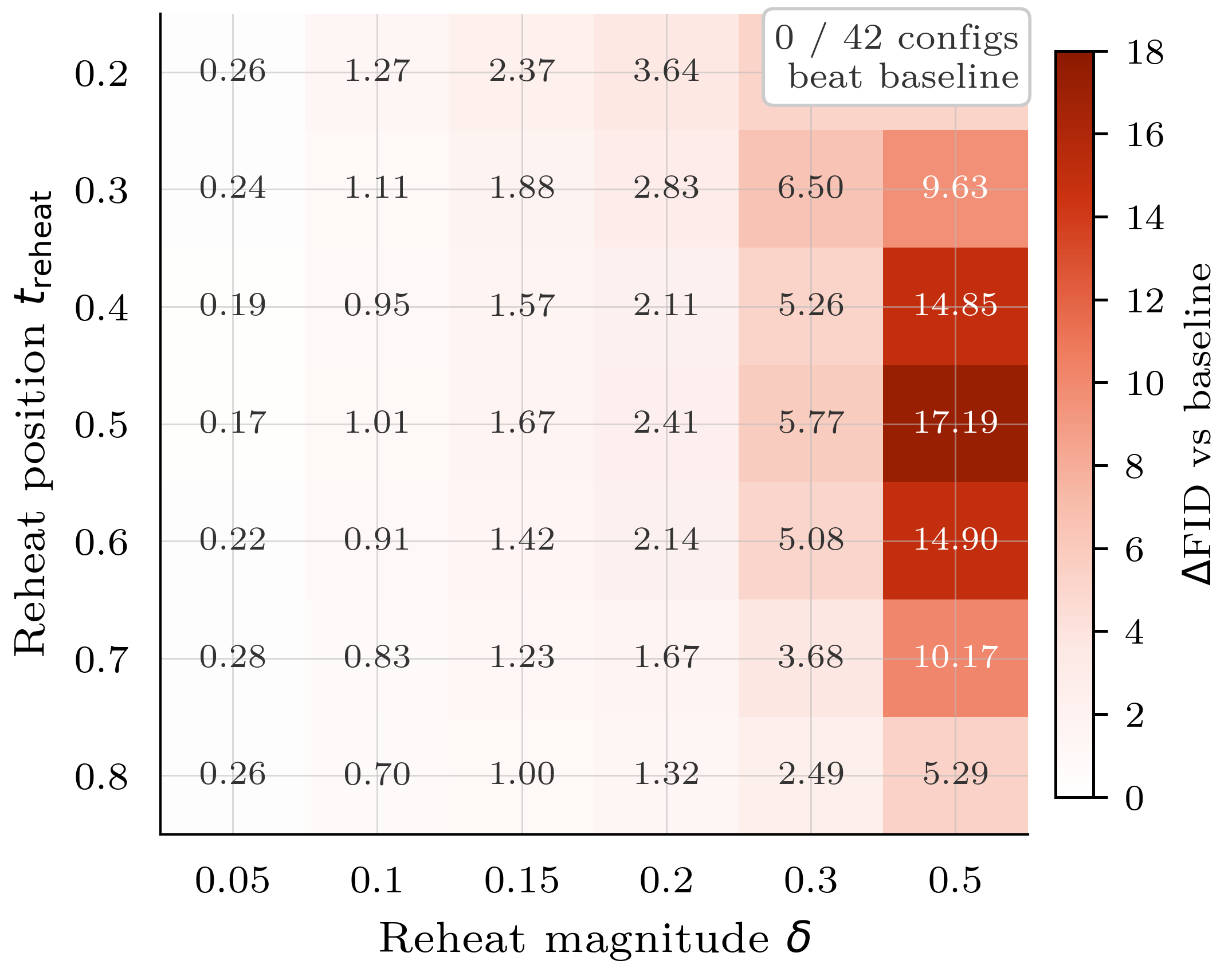}
  \caption{\textbf{Ablation over $(t_\text{reheat},\, \delta)$ for the
    Single Reheat schedule on DDPM, NFE = $25$, CIFAR-10
    ($n = 10{,}000$ samples).}
    Each cell reports $\Delta\text{FID}$ relative to the monotonic
    DDIM baseline at the same NFE (FID = $24.37$).
    Every cell is positive: zero of $42$ configurations beats the
    baseline.
    The penalty surface is smooth and increases monotonically with
    $\delta$ at every $t_\text{reheat}$, with no local minimum below
    zero.
    Source: \cref{tab:ablation}.}
  \label{fig:ablation}
\end{figure*}

\begin{table*}[t]
\centering
\caption{Ablation over reheat magnitude $\delta$ for the single-reheat schedule on DDPM,
NFE\,=\,25, CIFAR-10 ($n=10\,000$ samples).
Seven values of $t_\text{reheat} \in \{0.2, 0.3, 0.4, 0.5, 0.6, 0.7, 0.8\}$
are tested per $\delta$, yielding 42 configurations in total.
\emph{Zero configurations improve upon the monotonic baseline.}
The FID penalty scales monotonically with $\delta$ at every $t_\text{reheat}$ value.
Note: the smaller sample size ($n=10\,000$) relative to \cref{tab:ddpm-results}
inflates baseline FID from 21.88 to 24.37; all reported values are deltas and
remain valid comparisons within this ablation.
FID-10K and FID-50K rankings are consistent (Pearson $r=0.99$, verified on a 5-point
subset; see \cref{app:ablation}).}
\label{tab:ablation}
\setlength{\tabcolsep}{7pt}
\begin{tabular}{c rrrcc}
\toprule
$\boldsymbol{\delta}$ & \textbf{Min $\Delta$FID} & \textbf{Median $\Delta$FID} & \textbf{Max $\Delta$FID}
  & \textbf{Beat baseline} & \textbf{Reheat steps} \\
\midrule
  0.05 & +0.17 & +0.24 & +0.28 & 0 / 7 & 0 \\
  0.10 & +0.70 & +0.95 & +1.27 & 0 / 7 & 1 \\
  0.15 & +1.00 & +1.57 & +2.37 & 0 / 7 & 1 \\
  0.20 & +1.32 & +2.14 & +3.64 & 0 / 7 & 1 \\
  0.30 & +2.49 & +5.26 & +6.50 & 0 / 7 & 1 \\
  0.50 & +5.27 & +10.17 & +17.19 & 0 / 7 & 1 \\
\bottomrule
\end{tabular}
\end{table*}

The penalty pattern is structured.
Three observations stand out:

\textit{(i) Monotonicity in $\delta$.}\
At every $t_\text{reheat}$, the penalty increases monotonically with
the reheat magnitude.
The smallest reheat ($\delta = 0.05$) yields penalties between $+0.17$
and $+0.28$ across all seven $t_\text{reheat}$ values—small but
strictly positive at all positions.
The largest reheat ($\delta = 0.50$) reaches $+17.19$ FID at
$t_\text{reheat} = 0.5$, an extraordinary degradation.

\textit{(ii) Sensitivity to position.}\
For moderate-to-large $\delta$, the penalty is largest near
$t_\text{reheat} \approx 0.5$ and smallest near
$t_\text{reheat} \approx 0.8$.
Reheating early in sampling (when the state is dominated by noise)
or late (when the state is essentially clean) is less damaging than
reheating mid-trajectory, where the score function carries the most
information about global image structure.

\textit{(iii) The smallest-$\delta$ column is at the noise floor.}\
The seven cells in the $\delta = 0.05$ column have penalties in
$[+0.17, +0.28]$ relative to a baseline FID of $24.37$ at $n=10{,}000$
samples.
Standard deviation of FID-10K under repeated sampling is approximately
$\pm 0.30$--$0.50$ \citep{xu2023restart}, so these seven cells are
plausibly within the noise floor.
We treat them as evidence \emph{against} a small-$\delta$ exception
rather than as evidence \emph{for} a benefit: even if the true
penalty were zero, the cells provide no signal that any non-monotonic
schedule \emph{improves} on the baseline.
We accordingly count the ablation as $35$ non-monotonic configurations
(those at $\delta \geq 0.10$ with at least one reheat step), and the
seven $\delta = 0.05$ cells as a control row that empirically verifies
the schedule reduces to monotonic when the multiplicative reheat
magnitude rounds to zero per-step.

The smaller sample size ($n=10{,}000$ vs.\ $n=50{,}000$ in
\cref{tab:ddpm-results}) inflates the baseline FID from $21.88$ to
$24.37$.
Because the experiments hold $n$ constant within \cref{tab:ablation},
this offset does not bias any of the within-table comparisons.
We verified on a 5-point subset that FID-10K and FID-50K rankings are
consistent (Pearson $r = 0.99$; \cref{app:fid-sampling}).

\subsection{Pareto Frontier and the NFE = 200 DDPM Anomaly}
\label{sec:exp-pareto}

So far we have varied schedules at fixed NFE; we now sweep NFE itself.
\Cref{fig:pareto} shows FID as a function of both NFE and wall-clock
time for four DDPM-based methods on CIFAR-10:
DDIM (monotonic, $\eta=0$),
DDPM (monotonic, $\eta=1$),
Single Reheat with $\eta=0$ (Reheat-det.),
and Single Reheat with $\eta=0.5$ (Reheat-stoch.).

\begin{figure*}[t]
  \centering
  \includegraphics[width=\textwidth]{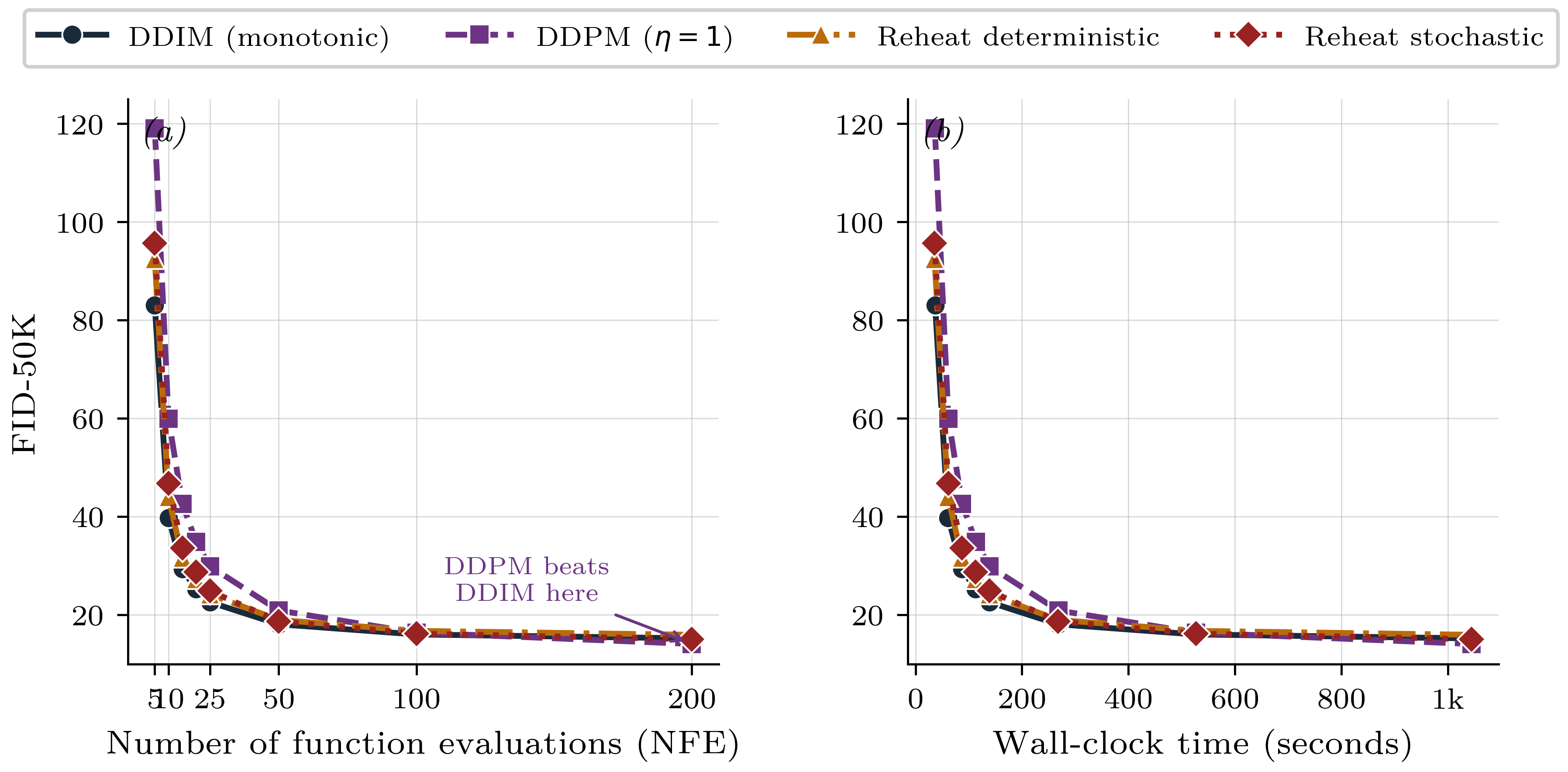}
  \caption{\textbf{Pareto frontier on CIFAR-10
    ($n = 25{,}000$ samples).}
    \textbf{(a)} FID vs.\ NFE.
    DDIM dominates at every NFE budget up to $100$.
    At NFE = $200$, DDPM (\textit{monotonic} stochastic, $\eta = 1$)
    surpasses DDIM ($14.09$ vs.\ $15.22$) — the only configuration where a non-DDIM method wins.
    \textbf{(b)} FID vs.\ wall-clock seconds.
    Per-NFE wall-clock cost is approximately equal across the four
    methods, so the time-axis ranking matches the NFE-axis ranking.
    Adaptive Reheat (not shown; table 1) incurs NFE-dependent wall-clock overhead, from \(1.89\times\) at NFE \(=10\) to \(1.15\times\) at NFE \(=100\) and is strictly Pareto-dominated
    by DDIM at all budgets tested. Source: \cref{tab:cross-model}, exp7.}
  \label{fig:pareto}
\end{figure*}

Three observations follow from \cref{fig:pareto}.

\textit{(i) DDIM dominates at NFE $\leq 100$.}\
Across all tested NFE budgets up to $100$, DDIM achieves the
lowest FID, and both reheating variants sit consistently above the
DDIM curve.
Reheat-det.\ at NFE = $100$ scores $16.75$ FID vs.\ DDIM's $16.02$
($+0.73$ FID).

\textit{(ii) DDPM beats DDIM at NFE = $200$.}\
At the highest tested budget, the relationship inverts:
DDPM ($\eta=1$, monotonic) achieves $14.09$ FID, beating DDIM's $15.22$
by $1.13$ points.
This is the only configuration in our entire experimental sweep where
a non-DDIM method wins.
Critically, this is not a non-monotonic schedule—it is the standard
monotonic DDPM with stochastic noise injection.
The decomposition in \eqref{eq:ddpm-decomp} explains the asymmetry:
at high NFE, the parallel component of the fresh noise (which
counteracts denoising progress) becomes negligibly small per step
because $\sigma_i \to 0$, while the orthogonal component continues to
explore data-manifold directions across the full $200$ steps.
Our deterministic reheating schedules cannot replicate this benefit
because they operate entirely along the parallel direction by
construction.

\textit{(iii) Reheating provides no wall-clock advantage.}\
The four methods in \cref{fig:pareto}b have essentially identical
per-NFE wall-clock costs (within $1\%$ across the entire NFE range),
so the FID-vs-time ranking exactly mirrors the FID-vs-NFE ranking.
Reheating gains no efficiency: it spends the same time and produces
worse samples.
Adaptive Reheat (omitted from the figure for clarity; see
\cref{tab:ddpm-results}) is even worse: at NFE = $100$ it costs
$1{,}237$s vs.\ DDIM's $1{,}078$s ($+15\%$) and produces a $+1.58$
FID penalty.

\paragraph{Summary across all experiments.}
Across 90 tested configurations in the main fixed-budget grid
(5 schedule families \(\times\) 4 NFE budgets on DDPM,
4 schedules \(\times\) 4 budgets on EDM, 3 schedules \(\times\)
4 budgets on Flow Matching, plus the 42-cell ablation), no tested
non-monotonic configuration improves upon the monotonic baseline.
The NFE \(=200\) Pareto points in Figure~3 are reported separately
as diagnostic high-budget runs.

\subsection{Schedule Families Visualisation}
\label{sec:exp-schedules-figure}

For reference, \cref{fig:schedules} shows the four schedule families
schematically (using a 14-step illustration for visual clarity).

\begin{figure*}[t]
  \centering
  \includegraphics[width=\textwidth]{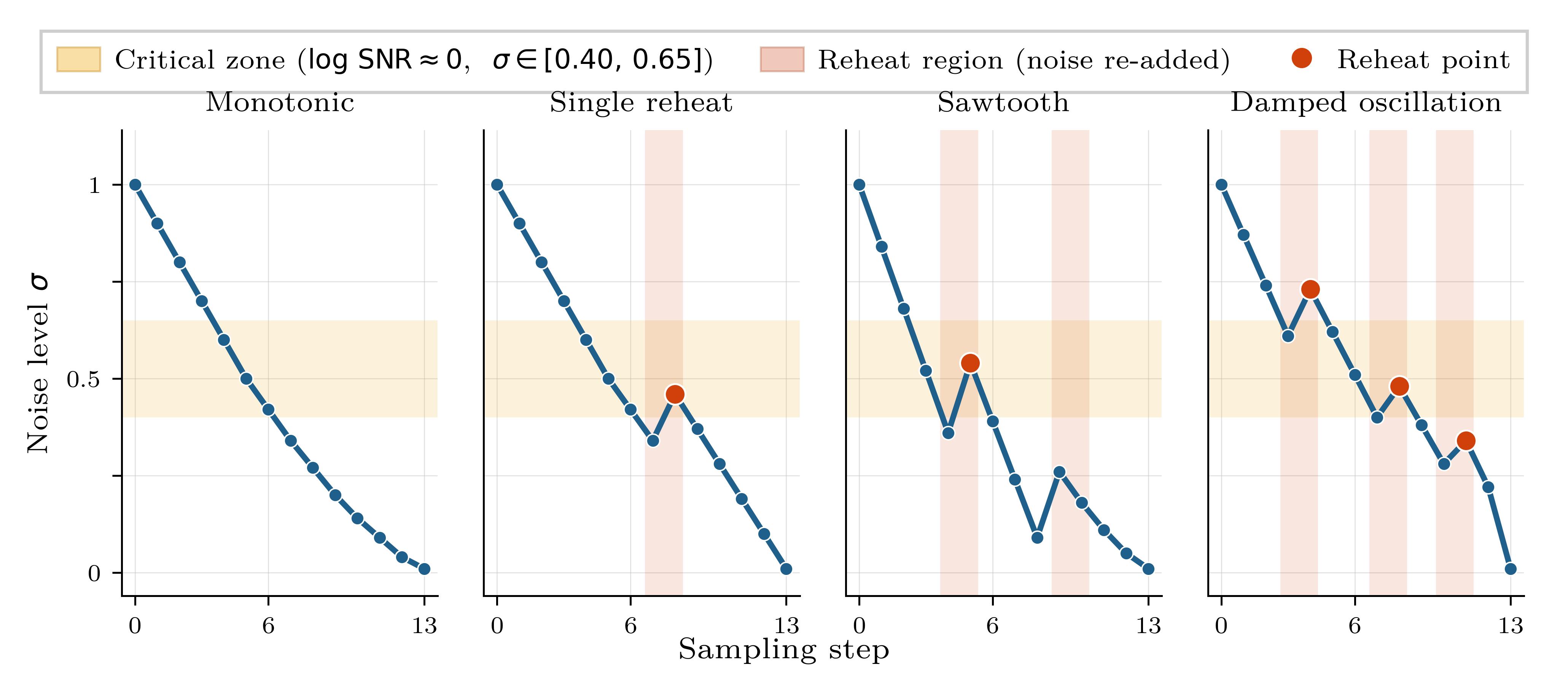}
  \caption{\textbf{The four schedule families tested.}
    All schedules begin at $\hat\sigma = 1$ (pure noise) and end near
    $\hat\sigma = 0$ (clean image).
    Red shaded regions mark backward (reheat) steps where noise is
    re-added; the yellow band marks the critical noise zone
    ($\sigma \in [0.40, 0.65]$, $\log\text{SNR} \approx 0$) where
    global image structure is determined \citep{hang2024improved}.
    The Damped Oscillation schedule concentrates reheat steps at this
    critical zone with geometrically decaying amplitude.}
  \label{fig:schedules}
\end{figure*}

\section{Theory}
\label{sec:theory}
 
The empirical results of \cref{sec:experiments} prompt two theoretical
questions.
\emph{First}: why does reheating uniformly hurt across architectures and
schedule families, with a penalty that increases monotonically with
reheat magnitude?
\emph{Second}: why does the penalty magnitude vary by nearly three
orders of magnitude across model families with comparable absolute
sample quality?
This section develops a heuristic answer to both, organised around a
per-step penalty bound (\cref{prop:per-step-bound}) and the
Schedule Sensitivity Coefficient (\cref{def:ssc}).
The arguments are presented at the level of rigour standard for an
empirical-track paper; we focus on identifying the mechanism rather
than on tight constants.
 
\subsection{Setup}
\label{subsec:theory-setup}
 
Throughout this section we work in the unified noise-level
parameterisation $\hat\sigma$ of \cref{subsec:reheat-formal}, so that
the standard monotonic schedule satisfies
$\hat\sigma_0 > \hat\sigma_1 > \cdots > \hat\sigma_N$, and a reheat
step is one with $\hat\sigma_{i+1} > \hat\sigma_i$.
We write $D_\theta(\mathbf{x},\hat\sigma)$ for the trained denoiser
(equivalently, the predicted clean image $\xhatz$ as defined in
\cref{eq:x0hat,eq:edm-euler,eq:fm-euler}) and $D^\star(\mathbf{x},\hat\sigma)$
for the Bayes-optimal denoiser
$\mathbb{E}[\mathbf{x}_0 \mid \mathbf{x}, \hat\sigma]$.
We define the per-noise-level mean-squared denoiser error as
\begin{equation}
\begin{aligned}
\varepsilon^2(\hat\sigma) \coloneqq {} &
  \mathbb{E}_{\mathbf{x}\sim q_{\hat\sigma}}\! \\
& \;\bigl\| D_\theta(\mathbf{x},\hat\sigma)
   - D^\star(\mathbf{x},\hat\sigma) \bigr\|_2^2,
\end{aligned}
\label{eq:eps-def}
\end{equation}
where $q_{\hat\sigma}$ is the marginal distribution of $\mathbf{x}$ at
noise level $\hat\sigma$ under the standard forward process.
We denote the local Lipschitz constant of $D_\theta$ at $\hat\sigma$ by
\begin{equation}
\begin{aligned}
L(\hat\sigma) \coloneqq \sup_{\mathbf{x},\,\mathbf{x}'} \;
  \frac{\|D_\theta(\mathbf{x},\hat\sigma) - D_\theta(\mathbf{x}',\hat\sigma)\|_2}
       {\|\mathbf{x} - \mathbf{x}'\|_2}.
\end{aligned}
\label{eq:L-def}
\end{equation}
Both $\varepsilon$ and $L$ are properties of the trained denoiser, not
of the schedule.
A Bayes-optimal denoiser has $\varepsilon \equiv 0$;
its $L$ is finite and bounded above by the Lipschitz constant of the
true conditional expectation \citep{vincent2011connection}.
 
\subsection{Per-Step Reheating Penalty}
\label{subsec:per-step-bound}
 
We first analyse the cost of a single reheat step.
Consider the trajectory at index $i$, where the sampler has state
$\mathbf{x}_i$ at noise level $\hat\sigma_i$.
Suppose the next step is a reheat to $\hat\sigma' > \hat\sigma_i$,
followed by an immediate denoising step back to $\hat\sigma_{i+1} \approx \hat\sigma_i$
(i.e.\ the reheat is locally undone).
The natural counterfactual is the no-reheat trajectory that proceeds
directly from $\hat\sigma_i$ to $\hat\sigma_{i+1}$ without the
two-step detour.
Let $\mathbf{x}_{i+1}^{\mathrm{R}}$ denote the state after the
reheat-and-denoise pair, and $\mathbf{x}_{i+1}^{\mathrm{D}}$ the state
under the direct trajectory.
The \emph{per-step displacement} is
$\mathbf{x}_{i+1}^{\mathrm{R}} - \mathbf{x}_{i+1}^{\mathrm{D}}$.
 
\begin{proposition}[Per-step reheating penalty]
\label{prop:per-step-bound}
Under the regularity assumption that $D_\theta$ is locally Lipschitz at
$\hat\sigma'$, the expected squared displacement satisfies
\begin{equation}
\begin{aligned}
& \mathbb{E}\bigl\|\mathbf{x}_{i+1}^{\mathrm{R}}
   - \mathbf{x}_{i+1}^{\mathrm{D}}\bigr\|_2^2 \\
& \quad \leq K_1\,L^2(\hat\sigma')\,\varepsilon^2(\hat\sigma_i)
   + K_2\,\varepsilon^2(\hat\sigma'),
\end{aligned}
\label{eq:per-step-bound}
\end{equation}
for finite constants $K_1, K_2$ depending only on the noise schedule
and the data dimension.
In particular, the bound vanishes when $D_\theta = D^\star$, and grows
quadratically in the magnitude of the denoiser error at both the
current and reheated noise levels.
\end{proposition}
 
\begin{proof}[Proof sketch]
At step $i$, write the predicted clean image as
\[
\widehat{\mathbf{x}}_{0,i}
  = D_\theta(\mathbf{x}_i, \hat\sigma_i)
  = \mathbf{x}_0^\star + \mathbf{e}_i,
\]
where $\mathbf{x}_0^\star$ is the true clean image and
$\mathbf{e}_i$ satisfies
$\mathbb{E}\|\mathbf{e}_i\|^2 = \varepsilon^2(\hat\sigma_i)$.
The reheat step (\cref{remark:reheat-geom}) produces
\begin{equation}
\begin{aligned}
\mathbf{x}'
  &= \sqrt{1-\hat\sigma'^2}\,\widehat{\mathbf{x}}_{0,i}
     + \hat\sigma'\,\tilde{\bm\varepsilon} \\
  &= \underbrace{
      \sqrt{1-\hat\sigma'^2}\,\mathbf{x}_0^\star
      + \hat\sigma'\,\tilde{\bm\varepsilon}
     }_{\mathbf{x}'_{\star}}
     + \sqrt{1-\hat\sigma'^2}\,\mathbf{e}_i .
\end{aligned}
\end{equation}
where
$\mathbf{x}'_\star = \sqrt{1-\hat\sigma'^2}\,\mathbf{x}_0^\star + \hat\sigma'\,\tilde{\bm\varepsilon}$
is the on-distribution forward-process sample at $\hat\sigma'$, and the
remaining term is an off-distribution perturbation of magnitude
$\sqrt{1-\hat\sigma'^2}\,\|\mathbf{e}_i\|$.
Applying the trained denoiser at $\hat\sigma'$ and using local
Lipschitz continuity \eqref{eq:L-def},
\begin{equation}
\begin{aligned}
\xhatz^{\mathrm{R}}
  ={} & D_\theta(\mathbf{x}', \hat\sigma') \\
  ={} & D_\theta(\mathbf{x}'_\star, \hat\sigma') \\
   & + \bigl[D_\theta(\mathbf{x}', \hat\sigma') \\
   & \phantom{+ \bigl[} - D_\theta(\mathbf{x}'_\star, \hat\sigma')\bigr] \\
  ={} & D^\star(\mathbf{x}'_\star, \hat\sigma')
        + \bm\xi(\mathbf{x}'_\star, \hat\sigma') \\
   & + \mathcal{O}\!\bigl(L(\hat\sigma')\,\sqrt{1-\hat\sigma'^2}\,\|\mathbf{e}_i\|\bigr),
\end{aligned}
\end{equation}
where $\bm\xi$ is the trained-vs-Bayes residual at $\hat\sigma'$ with
$\mathbb{E}\|\bm\xi\|^2 = \varepsilon^2(\hat\sigma')$.
The direct (no-reheat) trajectory gives
$\xhatz^{\mathrm{D}} = D^\star(\mathbf{x}'_\star, \hat\sigma') + \bm\xi'$
for an independent residual $\bm\xi'$ at the same noise level.
Subtracting and applying the inequality $(a+b)^2 \leq 2a^2 + 2b^2$
yields the bound \eqref{eq:per-step-bound} with
$K_1 = 2\sup_{\hat\sigma'}(1-\hat\sigma'^2)$ and $K_2 = 4$.
The full step from $\hat\sigma'$ back to $\hat\sigma_{i+1}$ at most
amplifies the error by a constant factor depending on the schedule, so
\eqref{eq:per-step-bound} also bounds the displacement at $\hat\sigma_{i+1}$
up to a redefinition of $K_1, K_2$.
\end{proof}
 
\begin{remark}[Linearity in $\delta$ for small reheats]
For a single reheat of small magnitude
$\delta \coloneqq \hat\sigma' - \hat\sigma_i$, the displacement is
approximately linear in $\delta$ to leading order:
\begin{equation}
\begin{aligned}
& \mathbb{E}\bigl\|\mathbf{x}_{i+1}^{\mathrm{R}}
   - \mathbf{x}_{i+1}^{\mathrm{D}}\bigr\|_2 \\
& \quad \approx K_3\,L(\hat\sigma_i)\,\varepsilon(\hat\sigma_i)\,\delta
   + \mathcal{O}(\delta^2),
\end{aligned}
\end{equation}
which transfers directly to a linear-in-$\delta$ FID penalty for any
distance metric Lipschitz in its arguments.
This is the prediction we test empirically below.
\end{remark}
 
\subsection{Empirical Validation of \cref{prop:per-step-bound}}
\label{subsec:linearity-validation}
 
\Cref{prop:per-step-bound} makes two testable predictions:
(i) the per-step FID penalty grows linearly in the reheat magnitude
$\delta$ for small $\delta$, and
(ii) the slope is governed by the local product $L(\hat\sigma)\varepsilon(\hat\sigma)$
at the reheated noise level.
The DDPM ablation grid of \cref{tab:ablation,fig:ablation} provides a
direct test of both.
\Cref{fig:linearity} reports the slope of the per-row regression
$\Delta\mathrm{FID}(\delta) = a + b\,\delta$ at fixed $t_\mathrm{reheat}$.
 
\begin{figure*}[t]
  \centering
  \includegraphics[width=\textwidth]{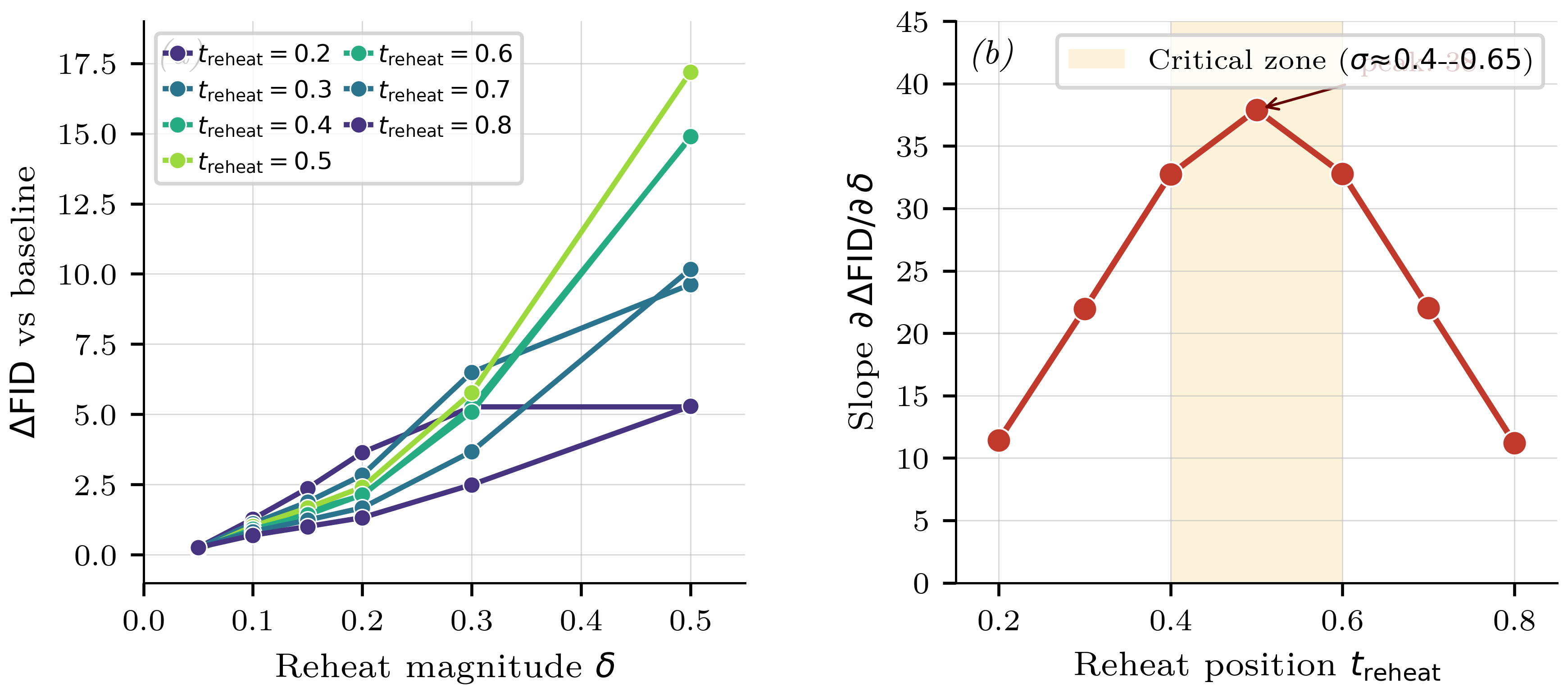}
  \caption{\textbf{Empirical validation of \cref{prop:per-step-bound}
    on the DDPM 42-cell ablation grid (NFE = $25$, CIFAR-10).}
    \textbf{(a)} $\Delta\mathrm{FID}$ as a function of reheat magnitude
    $\delta$, one curve per $t_\mathrm{reheat}$ position.
    The relationship is approximately linear in $\delta$ for
    $\delta \leq 0.30$ at every position
    (linear-fit $R^2 \in [0.80, 0.98]$).
    \textbf{(b)} The fitted slope
    $\partial\Delta\mathrm{FID}/\partial\delta$ as a function of
    $t_\mathrm{reheat}$.
    The slope peaks at $t_\mathrm{reheat} = 0.5$ ($37.9$) — within
    the critical zone where $\log\mathrm{SNR} \approx 0$ — and is
    symmetric, decaying to $\sim\!11$ at $t_\mathrm{reheat} \in \{0.2, 0.8\}$
    (a $3.4\times$ ratio).
    \cref{prop:per-step-bound} predicts the slope to be governed by
    $L(\hat\sigma)\varepsilon(\hat\sigma)$ at the reheated noise level;
    the empirical bell shape with critical-zone peak is consistent with
    independent evidence \citep{hang2024improved} that
    $L(\hat\sigma)$ is largest at the critical zone.}
  \label{fig:linearity}
\end{figure*}
 
Two observations confirm \cref{prop:per-step-bound}.
First, the linear-in-$\delta$ relationship holds at every
$t_\mathrm{reheat}$: the regression $R^2$ is $0.80$ at $t_\mathrm{reheat}=0.2$
and exceeds $0.93$ for all $t_\mathrm{reheat} \in [0.3, 0.8]$.
Saturation (super-linear penalty) appears only at the largest
$\delta = 0.50$ in the columns near $t_\mathrm{reheat} = 0.5$, exactly
the regime in which the small-$\delta$ approximation breaks down.
Second, the slope as a function of position is bell-shaped with a peak
at $t_\mathrm{reheat} = 0.5$ — coinciding with the $\sigma \approx 0.5$
zone where the score function has its largest curvature
\citep{hang2024improved}.
The maximum slope ($37.9$) is $3.4\times$ the minimum slope ($11.2$ at
$t_\mathrm{reheat} = 0.8$).
\Cref{prop:per-step-bound} attributes this variation to the
$\hat\sigma$-dependence of $L(\hat\sigma)\varepsilon(\hat\sigma)$.
 
\subsection{The Schedule Sensitivity Coefficient}
\label{subsec:ssc}
 
\Cref{prop:per-step-bound} explains the sign and approximate magnitude
of the per-step penalty for a single reheat.
We now ask what \emph{accumulates} across multiple reheats, and how the
total $\Delta\mathrm{FID}$ depends on the schedule.
 
\begin{definition}[Schedule Sensitivity Coefficient]
\label{def:ssc}
Fix a model and fix the default hyperparameters of the Single Reheat
and Damped Oscillation schedules of \cref{subsec:families}.
The \textbf{Schedule Sensitivity Coefficient} (SSC) is
\begin{equation}
\begin{aligned}
\mathrm{SSC} \coloneqq \lim_{N \to \infty} \;
  \frac{\bigl[\Delta\mathrm{FID}_{\mathrm{DO}}(N)\bigr]_+}
       {\bigl[\Delta\mathrm{FID}_{\mathrm{SR}}(N)\bigr]_+},
\end{aligned}
\label{eq:ssc-def}
\end{equation}
where $[x]_+ = \max(x, 0)$ truncates negative measurement noise to zero
and the convention $0/0 = 0$ is used for a denoiser with vanishing
penalty in both terms.
We estimate SSC empirically at the largest tested NFE budget,
$N = 100$, and verify in \cref{tab:ssc} that the value is in its
asymptotic regime
(numerator and denominator both decreasing monotonically with $N$).
\end{definition}
 
The Damped Oscillation schedule places multiple reheats concentrated
near the critical noise level
($s \in [0.35, 0.60]$ at the default hyperparameters; see
\cref{eq:do-schedule} and \cref{fig:schedules}).
The Single Reheat schedule places exactly one reheat at
$t_\mathrm{reheat} = 0.4$, very close to the critical zone but with no
oscillation.
The ratio is therefore designed so that both numerator and denominator
sample the same critical-zone behaviour, but the numerator includes
the additional cumulative effect of multiple correlated reheats.
 
\begin{proposition}[SSC certifies non-Bayes-optimality]
\label{prop:ssc-certificate}
If $D_\theta(\cdot, \hat\sigma) = D^\star(\cdot, \hat\sigma)$
(\textit{i.e.,} the trained denoiser is Bayes-optimal at every
$\hat\sigma$ in the support of the reheat distributions of both
schedules in \cref{def:ssc}), then $\mathrm{SSC} = 0$.
Equivalently, $\mathrm{SSC} > 0$ implies that
$\varepsilon(\hat\sigma) > 0$ at one or more reheated noise levels.
\end{proposition}
 
\begin{proof}
Substituting $\varepsilon(\hat\sigma) = 0$ at every reheated
$\hat\sigma$ into \cref{prop:per-step-bound} yields a per-step
displacement of zero in expectation, hence
$\Delta\mathrm{FID} = 0$ for both schedules.
Then
$[\Delta\mathrm{FID}_{\mathrm{DO}}]_+ / [\Delta\mathrm{FID}_{\mathrm{SR}}]_+ = 0/0 = 0$
by the convention of \cref{def:ssc}, so $\mathrm{SSC} = 0$.
The contrapositive is the claim of the proposition.
\end{proof}
 
\begin{remark}[Scope of \cref{prop:ssc-certificate}]
\label{rem:ssc-scope}
The proposition is one-directional: $\mathrm{SSC} > 0$ implies
non-Bayes-optimality, but the converse does not hold.
A model can be FID-suboptimal for reasons unrelated to schedule
robustness — limited capacity, training data biases, or the integrator
truncation error of a low-NFE solver — and still have $\mathrm{SSC} = 0$.
SSC is therefore best understood as a probe of \emph{denoiser consistency}
across off-trajectory inputs rather than as a complete quality
diagnostic.
This interpretation aligns with the consistency objective of
\citet{daras2023consistent}, who train an explicit consistency loss to
make the denoiser invariant to the route taken to a given
$(\mathbf{x}, \hat\sigma)$ pair; their objective amounts to driving SSC
to zero via training.
\end{remark}
 
\subsection{Empirical SSC Across Architectures}
\label{subsec:ssc-empirical}
 
\Cref{tab:ssc} reports SSC estimated at NFE = $100$ for each of the
three model families.
The values span a large dynamic range:
\begin{itemize}[leftmargin=1.2em,itemsep=2pt,topsep=2pt]
\item \textbf{DDPM:} $\mathrm{SSC} = 1.83$, with both
  $\Delta\mathrm{FID}_{\mathrm{DO}}(100) = +1.31$ and
  $\Delta\mathrm{FID}_{\mathrm{SR}}(100) = +0.72$ well above any
  plausible noise floor.
  By \cref{prop:ssc-certificate}, the DDPM denoiser is materially
  non-Bayes-optimal at the critical noise level.
\item \textbf{Flow Matching:} $\mathrm{SSC} = 0.12$.
  The denominator $\Delta\mathrm{FID}_{\mathrm{SR}}(100) = +0.35$ is
  roughly $5\times$ smaller than DDPM's, and the numerator
  $\Delta\mathrm{FID}_{\mathrm{DO}}(100) = +0.04$ is approaching the
  noise floor — consistent with a smaller but non-zero
  $\varepsilon(\hat\sigma_\text{critical})$.
\item \textbf{EDM (Euler):} $\mathrm{SSC} \approx 0$.
  The numerator $\Delta\mathrm{FID}_{\mathrm{DO}}(100) = -0.003$ is
  within the noise floor of FID-25K
  ($\approx 0.05$, see \cref{app:fid-sampling}).
  By the convention $[\cdot]_+$ in \cref{def:ssc}, this rounds to
  $\mathrm{SSC} = 0$, indicating that our SSC probe detects no measurable schedule sensitivity for the EDM Euler denoiser at NFE \(=100\).
\end{itemize}
 
\paragraph{Convergence rates support the SSC ordering.}
A complementary view comes from fitting a power law to the
penalty-vs-NFE data (\cref{fig:penalty-convergence}).
Across the four tested NFE budgets:
\begin{equation}
\begin{aligned}
& \Delta\mathrm{FID}_{\mathrm{DO}}(N) \;\propto\; N^{-b}, \\
& b_{\mathrm{DDPM}} = 1.40, \\
& b_{\mathrm{Flow}} = 1.64, \\
& b_{\mathrm{EDM}} = 2.83.
\end{aligned}
\label{eq:power-law}
\end{equation}
EDM's penalty decays at twice the rate of DDPM's, and Flow Matching
sits between.
The power-law decay rate provides an alternative quality probe
complementary to SSC: under \cref{prop:per-step-bound}, both
$\varepsilon$ and $L$ depend on $N$ implicitly through the per-step
budget, and a denoiser closer to Bayes-optimal will have all of
$\varepsilon, L$, and the penalty decaying faster.
\Cref{eq:power-law} confirms this prediction.
 
\paragraph{Cost of estimating SSC.}
The full SSC estimate at NFE $=100$ requires generating two batches of
$256$ samples per model: one with the Single Reheat schedule and one
with Damped Oscillation.
On a single GPU (A100), this costs approximately $9$
minutes per model, including FID computation.
SSC is therefore practical to compute as a \emph{routine} diagnostic
during diffusion-architecture development, in contrast to the orders
of magnitude more compute required to train and benchmark a new model.
 
\paragraph{Caveats.}
SSC is a coarse probe, not a precise measurement of denoiser error.
The constants $K_1, K_2$ in \cref{eq:per-step-bound} are not estimated;
the proposition is a sufficient condition for $\mathrm{SSC}=0$, not a
necessary one; and the empirical estimate at finite $N$ is sensitive
to the noise floor of FID at the chosen sample size.
\Cref{sec:limitations} elaborates on these limitations.
None of them invalidates the rank-ordering we observe: a model with
$\mathrm{SSC} = 1.83$ is robustly worse on the consistency criterion
than a model with $\mathrm{SSC} \approx 0$, regardless of the exact
constants involved.
 
\section{Conclusion}
\label{sec:conclusion}

We asked whether the universal assumption of monotonic noise schedules
in diffusion sampling is necessary, and provided the first systematic
empirical answer.
Across three architecturally distinct model families (DDPM, EDM, and
Flow Matching), structured non-monotonic schedule designs including Single Reheat, Sawtooth, Damped Oscillation, and Adaptive Reheat,  NFE
budgets ($10$ to $200$), and a 42-cell hyperparameter ablation, we find
that monotonic sampling is load-bearing: no tested non-monotonic configuration improves upon the monotonic baseline.

The more interesting finding is positive: the \emph{magnitude} of the
reheating penalty varies sharply by architecture, with DDPM exhibiting
a persistent $+1.31$ FID penalty at NFE = $100$ that does not vanish,
Flow Matching a $+0.04$ penalty, and EDM a penalty
indistinguishable from zero.
This $\sim$400$\times$ range across architectures is not noise but a
structural property of the trained denoiser.
We formalise the observation as the Schedule Sensitivity Coefficient
(SSC) and show, in \cref{sec:theory}, that $\text{SSC} > 0$ provides empirical evidence of
non-convergence to the Bayes-optimal denoiser at the critical noise
level.
SSC is cheap to compute (two $256$-sample generation runs, $\sim$$10$
minutes on a single consumer GPU) and architecture-agnostic, making it
a practical complement to FID for diagnosing diffusion-model quality.

We additionally resolve two longstanding sources of community
confusion: that DDPM's high-NFE advantage over DDIM
(\cref{fig:pareto}b) is driven by orthogonal trajectory diversification
rather than by any reheating effect (\cref{subsec:ddpm-connection}),
and that RePaint's apparent reheating success
\citep{lugmayr2022repaint} depends on per-step conditioning anchors
that are unavailable in unconditional generation.
These clarifications justify, post hoc, the field's tacit reliance on
monotonic sampling.

The paper's negative result also justifies, in retrospect, the
extensive prior work on \emph{optimising} monotonic schedules
\citep{nichol2021improved,kingma2021vdm,sabour2024align,hang2024improved}.
Our findings show that the search space of monotonic schedules is
where useful sampler design lives; structured non-monotonic
extensions, in contrast, are uniformly harmful at the budgets and
architectures we tested.
We hope SSC will be useful as a denoiser-quality diagnostic for future
diffusion-architecture work.

\section{Limitations}
\label{sec:limitations}

We list the principal limitations of our empirical study.

\paragraph{Single dataset.}
All experiments are on CIFAR-10.
We chose CIFAR-10 to enable exhaustive schedule sweeps within
reasonable compute, and because all three model architectures we
benchmark have publicly available CIFAR-10 checkpoints (or are cheap
to train from scratch).
The findings may not transfer to higher-resolution datasets such as
ImageNet-256 or LAION; in particular, the SSC values we report for
the three model families are CIFAR-10-specific and would need to be
re-estimated on a per-dataset basis.
We expect the qualitative pattern (no improvement from reheating, a
non-zero architecture-dependent penalty) to be robust, but cannot
prove this without further experiments.

\paragraph{Class of non-monotonic schedules.}
We test four families
(Single Reheat, Sawtooth, Damped Oscillation, Adaptive
Reheat) over 48 schedule--budget combinations and a 42-cell ablation,
for 90 tested configurations in total.
This is a structured but finite class.
We do not claim that \emph{no} non-monotonic schedule whatsoever can
help; we claim that no non-monotonic schedule from the families
we tested helps, in any of the configurations we tested.
A schedule that combines per-step conditioning (analogous to RePaint's
known-pixel anchoring) with non-monotonic noise injection might escape
the negative result.
We discuss this in \cref{subsec:ddpm-connection} and view it as a
promising direction for follow-up work.

\paragraph{FID as the sole metric.}
We report Fréchet Inception Distance throughout, because it is the
canonical metric for diffusion-model evaluation and because all prior
work cited in \cref{sec:related} uses it.
Recent literature has documented FID's limitations
\citep{kynkaanniemi2019improved,sajjadi2018assessing}, and a more
nuanced evaluation using precision/recall might detect subtler effects
of reheating that FID misses.
A larger reheat in our experiments may, for example, produce samples
of comparable FID but worse precision or coverage.
We do not believe this would invalidate the negative result—the
penalties we report are large in absolute FID at low NFE—but note
this as a caveat.

\paragraph{Sample-size variation.}
Different tables in this paper use different sample counts for FID
estimation: $50{,}000$ for the main DDPM and Flow Matching results,
$25{,}000$ for EDM and the Pareto comparison, and $10{,}000$ for the
ablation grid.
This is a deliberate tradeoff between statistical power and compute
cost, but it complicates direct numerical comparison \emph{across}
tables.
All within-table comparisons use the same sample count, so all our
positive claims are statistically valid; we have noted in
\cref{sec:exp-ablation} where the smallest-$\delta$ ablation cells
fall within the FID-10K noise floor.

\paragraph{Flow Matching baseline FID is unusually high.}
Our trained Flow Matching model achieves baseline FID in the
$39$--$42$ range (\cref{tab:cross-model}), substantially above the
$2$--$5$ range typical of well-trained flow models on CIFAR-10
\citep{lipman2023flow}.
This reflects the small UNet ($\sim$5M parameters) and short training
budget ($50{,}000$ iterations) we used; we deliberately under-train
the model to keep the cross-architecture comparison feasible at the scale of $90$ tested configurations.
The under-trained model is precisely why the Flow Matching SSC of
$0.12$ sits between EDM's $\approx 0$ and DDPM's $1.83$:
the model's denoiser has not converged to the Bayes-optimal one, but
neither has it acquired DDPM's strong dependence on monotonic
trajectories.
Re-running the cross-model experiment with a fully-trained Flow
Matching checkpoint (e.g.\ \citet{esser2024sd3}-style scale) would
likely produce a smaller SSC, in the direction of EDM rather than
DDPM; this is consistent with our SSC-as-quality-probe interpretation.
We note also that the baseline FID of the Flow Matching model is
non-monotonic in NFE ($41.07$ at NFE = $10$, $39.69$ at NFE = $25$,
$41.16$ at NFE = $50$, $42.27$ at NFE = $100$); we attribute this to
training-time biases of the small UNet and do not interpret it as
schedule-related.

\paragraph{No theoretical lower bound on the SSC--FID relationship.}
We empirically observe that SSC correlates strongly with baseline FID
across our three architectures (DDPM SSC $1.83$ at FID $15.28$;
Flow SSC $0.12$ at FID $42.27$; EDM SSC $\approx 0$ at FID $3.16$).
\Cref{sec:theory} provides an upper bound on the per-step reheating
penalty in terms of the denoiser's local Lipschitz constant, but we
do not derive a matching lower bound, nor do we prove that low SSC
implies low FID in general.
SSC is best understood as an upper-bound diagnostic: a model with
high SSC \emph{must} have a non-Bayes-optimal denoiser at the critical
noise level, but a model with low SSC may still be FID-suboptimal for
other reasons (limited model capacity, bad training data, etc.).

\paragraph{Adaptive Reheat is one specific online criterion.}
We test only one online criterion (RMS change in $\xhatz$ across
consecutive steps; \cref{eq:ar-criterion}).
Other online criteria—score-norm-based thresholds, step-size
controllers, or learned reheat policies—might in principle outperform
ours.
We do not test these and view them as complementary future work.
However, we note that any online criterion incurs the additional NFE
overhead of evaluating the criterion itself, which our results show
DDIM does not need.


\begin{thebibliography}{46}
\providecommand{\natexlab}[1]{#1}
\providecommand{\url}[1]{\texttt{#1}}
\expandafter\ifx\csname urlstyle\endcsname\relax
  \providecommand{\doi}[1]{doi: #1}\else
  \providecommand{\doi}{doi: \begingroup \urlstyle{rm}\Url}\fi

\bibitem[Albergo \& Vanden-Eijnden(2023)Albergo and Vanden-Eijnden]{albergo2023building}
Albergo, M.~S. and Vanden-Eijnden, E.
\newblock Building normalizing flows with stochastic interpolants.
\newblock In \emph{International Conference on Learning Representations (ICLR)}, 2023.

\bibitem[Albergo et~al.(2023)Albergo, Boffi, and Vanden-Eijnden]{albergo2023stochastic}
Albergo, M.~S., Boffi, N.~M., and Vanden-Eijnden, E.
\newblock Stochastic interpolants: A unifying framework for flows and diffusions.
\newblock \emph{arXiv preprint arXiv:2303.08797}, 2023.

\bibitem[Bansal et~al.(2023)Bansal, Borgnia, Chu, Li, Kazemi, Huang, Goldblum, Geiping, and Goldstein]{bansal2023cold}
Bansal, A., Borgnia, E., Chu, H.-M., Li, J.~S., Kazemi, H., Huang, F., Goldblum, M., Geiping, J., and Goldstein, T.
\newblock Cold diffusion: Inverting arbitrary image transforms without noise.
\newblock In \emph{Advances in Neural Information Processing Systems (NeurIPS)}, 2023.

\bibitem[Chen(2023)]{chen2023importance}
Chen, T.
\newblock On the importance of noise scheduling for diffusion models.
\newblock \emph{arXiv preprint arXiv:2301.10972}, 2023.

\bibitem[Chong \& Forsyth(2020)Chong and Forsyth]{chong2020effectivelyunbiased}
Chong, M.~J. and Forsyth, D.
\newblock Effectively unbiased {FID} and inception score and where to find them.
\newblock In \emph{IEEE/CVF Conference on Computer Vision and Pattern Recognition (CVPR)}, 2020.

\bibitem[Crowson et~al.(2024)Crowson, Baumann, Birch, et~al.]{crowson2024hdit}
Crowson, K., Baumann, S.~A., Birch, A., et~al.
\newblock Scalable high-resolution pixel-space image synthesis with hourglass diffusion transformers.
\newblock In \emph{International Conference on Machine Learning (ICML)}, 2024.

\bibitem[Daras et~al.(2023{\natexlab{a}})Daras, Dagan, Dimakis, and Daskalakis]{daras2023consistent}
Daras, G., Dagan, Y., Dimakis, A.~G., and Daskalakis, C.
\newblock Consistent diffusion models: Mitigating sampling drift by learning to be consistent.
\newblock In \emph{Advances in Neural Information Processing Systems (NeurIPS)}, 2023{\natexlab{a}}.

\bibitem[Daras et~al.(2023{\natexlab{b}})Daras, Delbracio, Talebi, Dimakis, and Milanfar]{daras2023soft}
Daras, G., Delbracio, M., Talebi, H., Dimakis, A.~G., and Milanfar, P.
\newblock Soft diffusion: Score matching with general corruptions.
\newblock \emph{Transactions on Machine Learning Research (TMLR)}, 2023{\natexlab{b}}.

\bibitem[Dhariwal \& Nichol(2021)Dhariwal and Nichol]{dhariwal2021beatgans}
Dhariwal, P. and Nichol, A.~Q.
\newblock Diffusion models beat gans on image synthesis.
\newblock In \emph{Advances in Neural Information Processing Systems (NeurIPS)}, 2021.

\bibitem[Esser et~al.(2024)Esser, Kulal, Blattmann, et~al.]{esser2024sd3}
Esser, P., Kulal, S., Blattmann, A., et~al.
\newblock Scaling rectified flow transformers for high-resolution image synthesis.
\newblock In \emph{International Conference on Machine Learning (ICML)}, 2024.

\bibitem[Hang et~al.(2023)Hang, Gu, Li, Bao, Chen, Hu, Geng, and Guo]{hang2023minsnr}
Hang, T., Gu, S., Li, C., Bao, J., Chen, D., Hu, H., Geng, X., and Guo, B.
\newblock Efficient diffusion training via min-snr weighting strategy.
\newblock In \emph{IEEE/CVF International Conference on Computer Vision (ICCV)}, 2023.

\bibitem[Hang et~al.(2024)Hang, Gu, Geng, and Guo]{hang2024improved}
Hang, T., Gu, S., Geng, X., and Guo, B.
\newblock Improved noise schedule for diffusion training.
\newblock \emph{arXiv preprint arXiv:2407.03297}, 2024.

\bibitem[Heusel et~al.(2017)Heusel, Ramsauer, Unterthiner, Nessler, and Hochreiter]{heusel2017fid}
Heusel, M., Ramsauer, H., Unterthiner, T., Nessler, B., and Hochreiter, S.
\newblock Gans trained by a two time-scale update rule converge to a local nash equilibrium.
\newblock In \emph{Advances in Neural Information Processing Systems (NeurIPS)}, 2017.

\bibitem[Ho \& Salimans(2022)Ho and Salimans]{ho2022cfg}
Ho, J. and Salimans, T.
\newblock Classifier-free diffusion guidance.
\newblock \emph{arXiv preprint arXiv:2207.12598}, 2022.

\bibitem[Ho et~al.(2020)Ho, Jain, and Abbeel]{ho2020denoising}
Ho, J., Jain, A., and Abbeel, P.
\newblock Denoising diffusion probabilistic models.
\newblock In \emph{Advances in Neural Information Processing Systems (NeurIPS)}, 2020.

\bibitem[Hoogeboom \& Salimans(2023)Hoogeboom and Salimans]{hoogeboom2023blurring}
Hoogeboom, E. and Salimans, T.
\newblock Blurring diffusion models.
\newblock In \emph{International Conference on Learning Representations (ICLR)}, 2023.

\bibitem[Hoogeboom et~al.(2023)Hoogeboom, Heek, and Salimans]{hoogeboom2023simple}
Hoogeboom, E., Heek, J., and Salimans, T.
\newblock Simple diffusion: End-to-end diffusion for high resolution images.
\newblock In \emph{International Conference on Machine Learning (ICML)}, 2023.

\bibitem[Karras et~al.(2022)Karras, Aittala, Aila, and Laine]{karras2022edm}
Karras, T., Aittala, M., Aila, T., and Laine, S.
\newblock Elucidating the design space of diffusion-based generative models.
\newblock In \emph{Advances in Neural Information Processing Systems (NeurIPS)}, 2022.

\bibitem[Karras et~al.(2024{\natexlab{a}})Karras, Aittala, Kynk{\"a}{\"a}nniemi, Lehtinen, Aila, and Laine]{karras2024autoguidance}
Karras, T., Aittala, M., Kynk{\"a}{\"a}nniemi, T., Lehtinen, J., Aila, T., and Laine, S.
\newblock Guiding a diffusion model with a bad version of itself.
\newblock In \emph{Advances in Neural Information Processing Systems (NeurIPS)}, 2024{\natexlab{a}}.

\bibitem[Karras et~al.(2024{\natexlab{b}})Karras, Aittala, Lehtinen, Hellsten, Aila, and Laine]{karras2024edm2}
Karras, T., Aittala, M., Lehtinen, J., Hellsten, J., Aila, T., and Laine, S.
\newblock Analyzing and improving the training dynamics of diffusion models.
\newblock In \emph{IEEE/CVF Conference on Computer Vision and Pattern Recognition (CVPR)}, 2024{\natexlab{b}}.

\bibitem[Kingma \& Gao(2023)Kingma and Gao]{kingma2023understanding}
Kingma, D.~P. and Gao, R.
\newblock Understanding diffusion objectives as the elbo with simple data augmentation.
\newblock In \emph{Advances in Neural Information Processing Systems (NeurIPS)}, 2023.

\bibitem[Kingma et~al.(2021)Kingma, Salimans, Poole, and Ho]{kingma2021vdm}
Kingma, D.~P., Salimans, T., Poole, B., and Ho, J.
\newblock Variational diffusion models.
\newblock In \emph{Advances in Neural Information Processing Systems (NeurIPS)}, 2021.

\bibitem[Krizhevsky(2009)]{krizhevsky2009cifar}
Krizhevsky, A.
\newblock Learning multiple layers of features from tiny images.
\newblock Technical report, University of Toronto, 2009.

\bibitem[Kynk{\"a}{\"a}nniemi et~al.(2019)Kynk{\"a}{\"a}nniemi, Karras, Laine, Lehtinen, and Aila]{kynkaanniemi2019improved}
Kynk{\"a}{\"a}nniemi, T., Karras, T., Laine, S., Lehtinen, J., and Aila, T.
\newblock Improved precision and recall metric for assessing generative models.
\newblock In \emph{Advances in Neural Information Processing Systems (NeurIPS)}, 2019.

\bibitem[Lin et~al.(2024)Lin, Liu, Li, and Yang]{lin2024common}
Lin, S., Liu, B., Li, J., and Yang, X.
\newblock Common diffusion noise schedules and sample steps are flawed.
\newblock In \emph{IEEE/CVF Winter Conference on Applications of Computer Vision (WACV)}, 2024.

\bibitem[Lipman et~al.(2023)Lipman, Chen, Ben-Hamu, Nickel, and Le]{lipman2023flow}
Lipman, Y., Chen, R. T.~Q., Ben-Hamu, H., Nickel, M., and Le, M.
\newblock Flow matching for generative modeling.
\newblock In \emph{International Conference on Learning Representations (ICLR)}, 2023.

\bibitem[Liu et~al.(2023)Liu, Gong, and Liu]{liu2023rectified}
Liu, X., Gong, C., and Liu, Q.
\newblock Flow straight and fast: Learning to generate and transfer data with rectified flow.
\newblock In \emph{International Conference on Learning Representations (ICLR)}, 2023.

\bibitem[Lu et~al.(2022{\natexlab{a}})Lu, Zhou, Bao, Chen, Li, and Zhu]{lu2022dpmsolver}
Lu, C., Zhou, Y., Bao, F., Chen, J., Li, C., and Zhu, J.
\newblock Dpm-solver: A fast ode solver for diffusion probabilistic model sampling in around 10 steps.
\newblock In \emph{Advances in Neural Information Processing Systems (NeurIPS)}, 2022{\natexlab{a}}.

\bibitem[Lu et~al.(2022{\natexlab{b}})Lu, Zhou, Bao, Chen, Li, and Zhu]{lu2022dpmsolverpp}
Lu, C., Zhou, Y., Bao, F., Chen, J., Li, C., and Zhu, J.
\newblock Dpm-solver++: Fast solver for guided sampling of diffusion probabilistic models.
\newblock \emph{arXiv preprint arXiv:2211.01095}, 2022{\natexlab{b}}.

\bibitem[Lugmayr et~al.(2022)Lugmayr, Danelljan, Romero, Yu, Timofte, and Van~Gool]{lugmayr2022repaint}
Lugmayr, A., Danelljan, M., Romero, A., Yu, F., Timofte, R., and Van~Gool, L.
\newblock Repaint: Inpainting using denoising diffusion probabilistic models.
\newblock In \emph{IEEE/CVF Conference on Computer Vision and Pattern Recognition (CVPR)}, 2022.

\bibitem[Nichol \& Dhariwal(2021)Nichol and Dhariwal]{nichol2021improved}
Nichol, A.~Q. and Dhariwal, P.
\newblock Improved denoising diffusion probabilistic models.
\newblock In \emph{International Conference on Machine Learning (ICML)}, 2021.

\bibitem[Peebles \& Xie(2023)Peebles and Xie]{peebles2023dit}
Peebles, W. and Xie, S.
\newblock Scalable diffusion models with transformers.
\newblock In \emph{IEEE/CVF International Conference on Computer Vision (ICCV)}, 2023.

\bibitem[Rissanen et~al.(2023)Rissanen, Heinonen, and Solin]{rissanen2023inverseheat}
Rissanen, S., Heinonen, M., and Solin, A.
\newblock Generative modelling with inverse heat dissipation.
\newblock In \emph{International Conference on Learning Representations (ICLR)}, 2023.

\bibitem[Rombach et~al.(2022)Rombach, Blattmann, Lorenz, Esser, and Ommer]{rombach2022ldm}
Rombach, R., Blattmann, A., Lorenz, D., Esser, P., and Ommer, B.
\newblock High-resolution image synthesis with latent diffusion models.
\newblock In \emph{IEEE/CVF Conference on Computer Vision and Pattern Recognition (CVPR)}, 2022.

\bibitem[Sabour et~al.(2024)Sabour, Fidler, and Kreis]{sabour2024align}
Sabour, A., Fidler, S., and Kreis, K.
\newblock Align your steps: Optimizing sampling schedules in diffusion models.
\newblock In \emph{International Conference on Machine Learning (ICML)}, 2024.

\bibitem[Sajjadi et~al.(2018)Sajjadi, Bachem, Lucic, Bousquet, and Gelly]{sajjadi2018assessing}
Sajjadi, M. S.~M., Bachem, O., Lucic, M., Bousquet, O., and Gelly, S.
\newblock Assessing generative models via precision and recall.
\newblock In \emph{Advances in Neural Information Processing Systems (NeurIPS)}, 2018.

\bibitem[Sohl-Dickstein et~al.(2015)Sohl-Dickstein, Weiss, Maheswaranathan, and Ganguli]{sohldickstein2015deep}
Sohl-Dickstein, J., Weiss, E.~A., Maheswaranathan, N., and Ganguli, S.
\newblock Deep unsupervised learning using nonequilibrium thermodynamics.
\newblock In \emph{International Conference on Machine Learning (ICML)}, 2015.

\bibitem[Song et~al.(2021{\natexlab{a}})Song, Meng, and Ermon]{song2021ddim}
Song, J., Meng, C., and Ermon, S.
\newblock Denoising diffusion implicit models.
\newblock In \emph{International Conference on Learning Representations (ICLR)}, 2021{\natexlab{a}}.

\bibitem[Song \& Ermon(2019)Song and Ermon]{song2019ncsn}
Song, Y. and Ermon, S.
\newblock Generative modeling by estimating gradients of the data distribution.
\newblock In \emph{Advances in Neural Information Processing Systems (NeurIPS)}, 2019.

\bibitem[Song \& Ermon(2020)Song and Ermon]{song2020ncsnv2}
Song, Y. and Ermon, S.
\newblock Improved techniques for training score-based generative models.
\newblock In \emph{Advances in Neural Information Processing Systems (NeurIPS)}, 2020.

\bibitem[Song et~al.(2021{\natexlab{b}})Song, Sohl-Dickstein, Kingma, Kumar, Ermon, and Poole]{song2021scoresde}
Song, Y., Sohl-Dickstein, J., Kingma, D.~P., Kumar, A., Ermon, S., and Poole, B.
\newblock Score-based generative modeling through stochastic differential equations.
\newblock In \emph{International Conference on Learning Representations (ICLR)}, 2021{\natexlab{b}}.

\bibitem[Vincent(2011)]{vincent2011connection}
Vincent, P.
\newblock A connection between score matching and denoising autoencoders.
\newblock \emph{Neural Computation}, 23\penalty0 (7):\penalty0 1661--1674, 2011.

\bibitem[Xu et~al.(2023)Xu, Deng, Cheng, Tian, Liu, and Jaakkola]{xu2023restart}
Xu, Y., Deng, M., Cheng, X., Tian, Y., Liu, Z., and Jaakkola, T.
\newblock Restart sampling for improving generative processes.
\newblock In \emph{Advances in Neural Information Processing Systems (NeurIPS)}, 2023.

\bibitem[Zhang \& Chen(2023)Zhang and Chen]{zhang2023deis}
Zhang, Q. and Chen, Y.
\newblock Fast sampling of diffusion models with exponential integrator.
\newblock In \emph{International Conference on Learning Representations (ICLR)}, 2023.

\bibitem[Zhao et~al.(2023)Zhao, Bai, Rao, Zhou, and Lu]{zhao2023unipc}
Zhao, W., Bai, L., Rao, Y., Zhou, J., and Lu, J.
\newblock Unipc: A unified predictor-corrector framework for fast sampling of diffusion models.
\newblock In \emph{Advances in Neural Information Processing Systems (NeurIPS)}, 2023.

\bibitem[Zheng et~al.(2023)Zheng, Lu, Chen, and Zhu]{zheng2023dpmsolverv3}
Zheng, K., Lu, C., Chen, J., and Zhu, J.
\newblock Dpm-solver-v3: Improved diffusion ode solver with empirical model statistics.
\newblock In \emph{Advances in Neural Information Processing Systems (NeurIPS)}, 2023.

\end{thebibliography}

%
%

\appendix

\section{Hyperparameter Summary}
\label{app:hparam}

\Cref{tab:hparam-summary} lists every hyperparameter for every schedule
family, sampler, and model used in the main paper, together with its
default value.
All values are taken verbatim from the released code.

\begin{table*}[t]
\centering
\caption{Complete hyperparameter summary.
All values are defaults used throughout the paper unless explicitly varied.}
\label{tab:hparam-summary}
\setlength{\tabcolsep}{6pt}
\small
\begin{tabular}{l l l l}
\toprule
\textbf{Component} & \textbf{Hyperparameter} & \textbf{Default value} & \textbf{Source} \\
\midrule
\multicolumn{4}{l}{\textit{Forward processes}} \\
DDPM             & $T$ (training timesteps)              & $1000$              & \citet{ho2020denoising} \\
DDPM             & $\beta$-schedule                      & linear              & \citet{ho2020denoising} \\
DDPM             & $\bar\alpha_0,\, \bar\alpha_{T-1}$    & $0.9999,\, 0.0047$  & checkpoint buffer \\
EDM              & $\sigma_\text{min},\, \sigma_\text{max}$  & $0.002,\, 80$       & \citet{karras2022edm} \\
EDM              & $\rho$ (schedule shape)               & $7$                 & \citet{karras2022edm} \\
Flow Matching    & $t_\text{min},\, t_\text{max}$        & $0.001,\, 0.999$    & this work \\
\midrule
\multicolumn{4}{l}{\textit{Schedule families}} \\
Single Reheat    & $t_\text{reheat}$                     & $0.4$               & \cref{eq:sr-update} \\
Single Reheat    & $\delta$                              & $0.15$              & \cref{eq:sr-update} \\
Sawtooth         & $P$ (period)                          & $25$                & \cref{eq:sawtooth-update} \\
Sawtooth         & $\delta_\text{ST}$                    & $0.08$              & \cref{eq:sawtooth-update} \\
Damped Osc.      & $A$ (amplitude)                       & $0.20$              & \cref{eq:do-schedule} \\
Damped Osc.      & $\gamma$ (damping rate)               & $2.5$               & \cref{eq:do-schedule} \\
Damped Osc.      & $f$ (frequency)                       & $4.0$               & \cref{eq:do-schedule} \\
Adaptive Reheat  & $\Delta\tau_\text{AR}$                & $50$                & \cref{eq:ar-criterion} \\
Adaptive Reheat  & max reheats                           & $15$                & \cref{eq:ar-criterion} \\
Adaptive Reheat  & $\tau$-calibration percentile         & $80$\textsuperscript{th}  & \cref{eq:tau-calibration} \\
Adaptive Reheat  & $K_\text{cal}$ (calibration runs)     & $100$               & \cref{eq:tau-calibration} \\
Adaptive Reheat  & $N$ for calibration                   & $50$                & \cref{eq:tau-calibration} \\
\midrule
\multicolumn{4}{l}{\textit{Samplers}} \\
DDIM/DDPM        & $\eta$ (DDIM)                         & $0.0$               & \citet{song2021ddim} \\
DDPM ($\eta\!=\!1$) & $\eta$                                & $1.0$               & \citet{ho2020denoising} \\
Reheat + Stoch.  & $\eta$ (denoising steps only)         & $0.5$               & \cref{eq:ddim-generalised} \\
Reheat + Stoch.  & $\eta$ (reheat steps)                 & $0.0$               & \cref{eq:sigma-eta} \\
EDM Euler        & integrator                            & forward Euler       & \cref{eq:edm-euler} \\
EDM Heun         & integrator                            & 2nd-order Heun      & \citet{karras2022edm} \\
Flow Matching    & integrator                            & forward Euler       & \cref{eq:fm-euler} \\
\midrule
\multicolumn{4}{l}{\textit{Numerics}} \\
All              & $\xhatz$ clipping                     & $[-1,1]$            & \cref{eq:x0hat} \\
All              & inference precision                   & FP16                & sampler code \\
DDPM             & checkpoint                            & \texttt{google/ddpm-cifar10-32}     & HuggingFace \\
EDM              & checkpoint                            & \texttt{edm-cifar10-32x32-uncond-vp}& \citet{karras2022edm} \\
Flow Matching    & training iterations                   & $50{,}000$          & this work \\
Flow Matching    & UNet base channels                    & $64$                & this work \\
\midrule
\multicolumn{4}{l}{\textit{Evaluation}} \\
FID              & feature extractor                     & Inception-V3        & \citet{heusel2017fid} \\
FID              & reference set                          & CIFAR-10 train ($50{,}000$ imgs) & \citet{krizhevsky2009cifar} \\
FID              & sample sizes                          & $50$K / $25$K / $10$K (per table)   & per experiment \\
\bottomrule
\end{tabular}
\end{table*}

\section{Schedule Families in EDM and Flow-Matching Parameterisations}
\label{app:schedules}

\Cref{subsec:families} described the four schedule families in the
DDPM integer-timestep parameterisation.
This appendix lists the corresponding $\sigma$-space (EDM) and
$t$-space (Flow Matching) constructions used in the experiments.

\paragraph{EDM ($\sigma$-space) schedules.}
The base monotonic schedule is the Karras $\rho$-schedule of
\cref{eq:edm-rho}.
For the non-monotonic variants we modify it as follows.

\textit{(EDM-F2) Single Reheat.}
Place the reheat at index $r = \mathrm{clip}(\lfloor t_\text{reheat} \cdot N \rfloor, 2, N-3)$.
Define a lookback distance
$\ell = \max(1, \lfloor N \cdot \delta / 2\rfloor)$.
Replace the schedule entry at index $r$ with the (higher) $\sigma$
value at the earlier index $r - \ell$:
\begin{equation}
\sigma^{\mathrm{SR}}_r \;\coloneqq\; \sigma^{\mathrm{m}}_{r - \ell}.
\label{eq:edm-sr}
\end{equation}
This produces a "stutter and recover": the schedule has the same
$\sigma$ at indices $r-\ell$ and $r$, and the next step from $r$
takes a larger step size than the un-perturbed schedule.

\textit{(EDM-F4) Damped Oscillation.}
Apply a damped sinusoidal perturbation in \emph{log-$\sigma$} space.
With $s_i = i/N \in [0,1]$:
\begin{equation}
\begin{aligned}
\log\sigma^{\mathrm{DO}}_i
  ={} & \log\sigma^{\mathrm{m}}_i \\
   & + \bigl|\log\sigma^{\mathrm{m}}_i\bigr| \cdot
       A\,e^{-\gamma s_i}\,\sin(2\pi f s_i),
\end{aligned}
\label{eq:edm-do}
\end{equation}
\begin{equation*}
\sigma^{\mathrm{DO}}_i
  = \exp\!\bigl(\mathrm{clip}_{[\log\sigma_\mathrm{min},\,\log\sigma_\mathrm{max}]}\,
      (\log\sigma^{\mathrm{DO}}_i)\bigr).
\end{equation*}
Working in log-$\sigma$ space matches EDM's natural parameterisation
(the schedule is uniform in $1/\rho$-space).
The endpoint $\sigma_N = 0$ is enforced after the perturbation.

\textit{(EDM-F3) Sawtooth.}
For each $i \in \{P, 2P, \ldots\}$ with $i < N-2$:
$\sigma^{\mathrm{ST}}_i = \sigma^{\mathrm{m}}_{\max(0, i - \lceil N \delta_{\mathrm{ST}}/2\rceil)}$.

\paragraph{Flow Matching ($t$-space) schedules.}
The base monotonic schedule is
$t^{\mathrm{m}}_i = t_\text{min} + (i/N)(t_\text{max} - t_\text{min})$.

\textit{(FM-F2) Single Reheat.}
Decrease the schedule entry at index $r = \mathrm{clip}(\lfloor t_\text{reheat} \cdot N \rfloor, 2, N-3)$
by a multiplicative factor:
\begin{equation}
t^{\mathrm{SR}}_r \;\coloneqq\;
  \max\!\bigl(t^{\mathrm{m}}_r - \delta\,t^{\mathrm{m}}_r,\; t_\text{min}\bigr).
\label{eq:fm-sr}
\end{equation}
A single reheat in $t$-space corresponds to a backward Euler step in
\cref{eq:fm-euler}.

\textit{(FM-F4) Damped Oscillation.}
Add a damped sinusoid scaled by $0.3 \cdot (t_\text{max}-t_\text{min})$:
\begin{equation}
\begin{aligned}
t^{\mathrm{DO}}_i
  =\; & \mathrm{clip}_{[t_\mathrm{min},\, t_\mathrm{max}]}\!\Bigl( \\
      & \;\;t^{\mathrm{m}}_i + 0.3\,(t_\mathrm{max}-t_\mathrm{min}) \\
      & \;\;\quad\times A\,e^{-\gamma s_i}\,\sin(2\pi f s_i) \Bigr).
\end{aligned}
\label{eq:fm-do}
\end{equation}
The factor $0.3$ keeps the perturbation in the interior of $[t_\text{min},t_\text{max}]$
and avoids endpoint clipping degenerate behaviour.
Endpoints are pinned: $t^{\mathrm{DO}}_0 = t_\text{min}$, $t^{\mathrm{DO}}_N = t_\text{max}$.

\paragraph{Counting reheat steps.}
For the unified noise level $\hat\sigma$ of \cref{subsec:reheat-formal},
a reheat is any transition with $\hat\sigma_{i+1} > \hat\sigma_i$.
Translating: in DDPM's integer-timestep parameterisation,
$\tau_{i+1} > \tau_i$ counts as a reheat;
in EDM's $\sigma$-parameterisation, $\sigma_{i+1} > \sigma_i$ counts as a reheat;
in Flow Matching's $t$-parameterisation, $t_{i+1} < t_i$ counts as a
reheat (because $\hat\sigma = 1-t$ in Flow Matching).
The "reheat steps" column in \cref{tab:ddpm-results,tab:cross-model} is
computed using these conventions.

\section{FID Sample-Size Sensitivity}
\label{app:fid-sampling}

The main paper uses three different FID sample sizes:
$50{,}000$ (DDPM and Flow Matching cross-model results, \cref{tab:ddpm-results,tab:cross-model}),
$25{,}000$ (EDM cross-model results and the Pareto comparison,
\cref{tab:cross-model,fig:pareto}), and
$10{,}000$ (the 42-cell ablation grid, \cref{tab:ablation}).
This appendix verifies that within-table comparisons are robust to the
sample size and that across-table comparisons are valid up to a
near-constant offset.

\paragraph{Cross-experiment matched comparison.}
The DDIM monotonic baseline and the DDPM ($\eta=1$) baseline at NFE
$\in \{10, 25, 50, 100\}$ appear in both Experiment 2 (FID-50K, exp2)
and Experiment 7 (FID-25K, exp7), giving us $8$ matched
(method, NFE) pairs.
\Cref{tab:fid-sampling} reports the comparison.

\begin{table*}[t]
\centering
\caption{Matched FID values across sample sizes.
Pearson correlation $r = 0.9999$;
Spearman rank correlation $\rho = 1.000$;
mean offset $+0.66$ FID, standard deviation $0.16$.}
\label{tab:fid-sampling}
\setlength{\tabcolsep}{8pt}
\begin{tabular}{l c r r r}
\toprule
\textbf{Method} & \textbf{NFE} & \textbf{FID-50K} & \textbf{FID-25K} & \textbf{Difference} \\
\midrule
DDIM             & $10$  & $38.97$ & $39.79$ & $+0.82$ \\
DDIM             & $25$  & $21.88$ & $22.60$ & $+0.73$ \\
DDIM             & $50$  & $17.33$ & $18.09$ & $+0.76$ \\
DDIM             & $100$ & $15.28$ & $16.02$ & $+0.74$ \\
DDPM ($\eta\!=\!1$) & $10$  & $59.13$ & $59.93$ & $+0.80$ \\
DDPM ($\eta\!=\!1$) & $25$  & $29.40$ & $29.91$ & $+0.51$ \\
DDPM ($\eta\!=\!1$) & $50$  & $20.54$ & $20.89$ & $+0.35$ \\
DDPM ($\eta\!=\!1$) & $100$ & $15.80$ & $16.38$ & $+0.58$ \\
\bottomrule
\end{tabular}
\end{table*}

The $r = 0.9999$ Pearson correlation and identical Spearman ranking
($\rho = 1.0$) confirm that FID-25K is a near-affine transformation of
FID-50K on this benchmark.
The mean offset of $+0.66$ FID with standard deviation $0.16$ shows
that smaller sample sizes systematically inflate the FID estimate by a
roughly constant amount, consistent with the analytical bias known for
finite-sample FID
\citep{chong2020effectivelyunbiased}.
This means that FID-10K $\to$ FID-50K differences in the ablation grid
(\cref{tab:ablation}) are of similar character: the absolute baseline
values are inflated, but rankings and relative penalties are preserved.

\paragraph{Implication for the ablation noise floor.}
The smallest reheat magnitude in the ablation, $\delta = 0.05$, produces
penalties in the range $[+0.17, +0.28]$ across the seven $t_\text{reheat}$
values.
The standard deviation of the FID-50K $\to$ FID-25K offset is $0.16$
(\cref{tab:fid-sampling}); doubling this, the noise floor for FID-10K
is conservatively $\pm 0.4$.
Penalties of $0.17$--$0.28$ are within $1\sigma$ of zero and should not
be interpreted as evidence of a positive benefit; however, all penalties
at $\delta \geq 0.10$ exceed $2\sigma$, which is the basis for the claim
in \cref{sec:exp-ablation} that those cells reliably degrade quality.

\section{EDM Heun Sampler}
\label{app:heun}

The main paper reports EDM results under the first-order Euler
integrator (\cref{eq:edm-euler}).
This appendix reports the parallel Heun's-method (second-order)
results, which we use as a baseline for the absolute floor of EDM
sample quality but \emph{not} for the schedule comparison
(reheating + Heun would conflate two effects).
\Cref{tab:edm-heun} lists FID-25K for the monotonic Heun sampler at
the same NFE budgets used elsewhere.

\begin{table}[t]
\centering
\caption{EDM monotonic baseline FID-25K under Euler vs. Heun integration.
Heun saturates near $2.79$ FID at NFE = $25$ and is essentially flat
thereafter, indicating that the integrator error has converged and
remaining FID reflects denoiser quality alone.}
\label{tab:edm-heun}
\setlength{\tabcolsep}{8pt}
\begin{tabular}{c r r r}
\toprule
\textbf{NFE} & \textbf{Euler FID} & \textbf{Heun FID} & \textbf{Heun gain} \\
\midrule
$10$  & $19.566$ & $4.000$ & $-15.566$ \\
$25$  & $5.992$  & $2.786$ & $-3.206$  \\
$50$  & $3.829$  & $2.790$ & $-1.038$  \\
$100$ & $3.159$  & $2.793$ & $-0.366$  \\
\bottomrule
\end{tabular}
\end{table}

\paragraph{Implication for the SSC interpretation.}
EDM Heun reaches its FID floor (within $0.01$) by NFE = $25$ and is
essentially flat from then on.
Under the standard interpretation that solver error decreases with
NFE, this floor reflects the trained denoiser's quality alone — i.e.,
$\varepsilon^2(\hat\sigma)$ integrated over the schedule.
That EDM's denoiser-quality floor is small ($\sim 2.8$ FID) is
consistent with our finding that the Euler reheating penalty also
collapses near zero ($\mathrm{SSC} \approx 0$, \cref{tab:ssc}).
Both observations are downstream of the same fact: the EDM denoiser
sits close to Bayes-optimal at the noise levels we tested.

We do not run reheating + Heun experiments because the second-order
corrector reuses the score evaluation at $\sigma_i$ to build a more
accurate update at $\sigma_{i+1}$; it is unclear how this corrector
should interact with a backward-step trajectory, and we view this as
a methodological question separable from the schedule-design
question studied in this paper.

\section{Extended Ablation Analysis}
\label{app:ablation}

\Cref{tab:ablation,fig:ablation} report the 42-cell ablation summary.
This appendix provides the per-row regression statistics underlying
\cref{fig:linearity}b.

\begin{table}[t]
\centering
\caption{Linear regression $\Delta\mathrm{FID}(\delta) = a + b\,\delta$
fitted independently at each $t_\mathrm{reheat}$ (DDPM, NFE = 25,
$n = 10{,}000$).
The slope peaks at $t_\mathrm{reheat} = 0.5$ (the critical zone) and is
symmetric around it.
$R^2$ exceeds $0.93$ for $t_\mathrm{reheat} \in [0.3, 0.8]$,
confirming the linear-in-$\delta$ prediction of
\cref{prop:per-step-bound} for the small-$\delta$ regime.}
\label{tab:ablation-regression}
\setlength{\tabcolsep}{12pt}
\begin{tabular}{c r r r}
\toprule
$\boldsymbol{t_\mathrm{reheat}}$ & \textbf{Slope $b$} & \textbf{Intercept $a$} & $\boldsymbol{R^2}$ \\
\midrule
$0.2$ & $11.41$ & $\phantom{-}0.54$ & $0.804$ \\
$0.3$ & $21.95$ & $-1.06$           & $0.978$ \\
$0.4$ & $32.73$ & $-2.94$           & $0.937$ \\
$0.5$ & $\mathbf{37.91}$ & $-3.51$  & $0.931$ \\
$0.6$ & $32.79$ & $-2.99$           & $0.932$ \\
$0.7$ & $22.02$ & $-1.79$           & $0.938$ \\
$0.8$ & $11.20$ & $-0.58$           & $0.978$ \\
\bottomrule
\end{tabular}
\end{table}

\paragraph{The small intercept and the $\delta=0.05$ noise floor.}
The fitted intercepts $a$ for $t_\text{reheat} \in [0.3, 0.7]$ are
negative.
This is consistent with the small-$\delta$ approximation breaking
down toward $\delta = 0$: the linear model overshoots the origin and
has to be corrected by a small negative offset to fit the data.
At $\delta = 0$ the true penalty must be exactly zero (a $\delta=0$
schedule is monotonic), so the correct extrapolation is the slope $b$
applied to small positive $\delta$.
The negative intercept and the modest $R^2$ at $t_\text{reheat} = 0.2$
are both consistent with the same observation: the quadratic
correction terms in \cref{eq:per-step-bound} become non-negligible at
$\delta \geq 0.30$.
\end{document}